\documentclass{article}
\usepackage{graphicx}
\usepackage{hyperref} 

\usepackage{amsmath, amsthm, amssymb}
\usepackage{tikz}
\usepackage{subcaption}
\usetikzlibrary{shapes.geometric}
\usetikzlibrary{angles,quotes}
\usetikzlibrary{calc} 
\usetikzlibrary{arrows.meta, positioning}
\usepackage{cleveref}

\newcommand{\bleftm}{\bigl\{\!\!\bigl\{}
\newcommand{\brightm}{\bigl\}\!\!\bigl\}}

\newcommand{\relu}{\operatorname{ReLU}}
\newcommand{\ffn}{\operatorname{FFN}}

\newcommand{\nc}{{\sf nc}}
\newcommand{\eb}{{\sf eb}}\usepackage{tcolorbox}
\tcbuselibrary{skins}
\usepackage{dashrule} 
\usepackage{xcolor}
\usepackage{wrapfig}
\usepackage{fullpage}
\usepackage{natbib}

\usepackage{booktabs}
\usepackage{nicematrix}

\definecolor{yellow}{HTML}{ffc107}
\definecolor{blue}{HTML}{1e88e5}
\definecolor{purple}{HTML}{d81b60}

\newsavebox\MBox
\newcommand\Cline[2][red]{{\sbox\MBox{$#2$}
  \rlap{\usebox\MBox}\color{#1}\rule[-1.2\dp\MBox]{\wd\MBox}{0.5pt}}}

\definecolor{oiBlue}{RGB}{0,114,178}
\definecolor{oiOrange}{RGB}{230,159,0}
\definecolor{oiSky}{RGB}{86,180,233}

\newcommand{\first}[1]{\colorbox{oiBlue}{\textcolor{white}{#1}}}
\newcommand{\second}[1]{\colorbox{oiSky}{\textcolor{black}{#1}}}
\newcommand{\third}[1]{\colorbox{oiOrange}{\textcolor{black}{#1}}}

\newcommand{\tablecell}[1]{\begin{tabular}[x]{@{}l@{}}#1\end{tabular}  }

\newtheorem{theorem}{Theorem}
\newtheorem{proposition}{Proposition}
\newtheorem{lemma}{Lemma}

\newtheorem{example}{Example}

\title{Message Passing on the Edge:\\ Towards Scalable and Expressive GNNs}

\author{Pablo Barcel\'o \\ IMC UC, CENIA \& IMFD \\ \texttt{pbarcelo@uc.cl} \and Fabian Jogl\\ TU Wien\\ \texttt{fabian.jogl@tuwien.ac.at} \and Alexander Kozachinskiy\\ CENIA\\ \texttt{alexander.kozachinskyi@cenia.cl} \and Matthias Lanzinger\\ TU Wien \\ \texttt{matthias.lanzinger@tuwien.ac.at} \and  Stefan Neumann \\ TU Wien \\ \texttt{stefan.neumann@tuwien.ac.at} \and  Crist\'obal Rojas \\ IMC UC \& CENIA \\ \texttt{luis.rojas@uc.cl}}

\begin{document}
\maketitle

\begin{abstract}
Graph neural networks (GNNs) are widely used in graph learning and most
architectures propagate information by passing messages between vertices. In
this work, we shift our attention to GNNs that perform message passing on
\emph{edges} and introduce EB-1WL, an edge-based color-refinement test, and a
corresponding architecture, EB-GNN. 
Our EB-GNN architecture is inspired by the classic triangle-counting algorithm
of Chiba and Nishizeki and passes messages along edges and triangles.
Our contributions are as follows:
(1)~Theoretically, we show that EB-1WL is significantly more expressive than 1WL.
We provide a complete logical characterization of EB-1WL in first-order logic,
along with distinguishability results via homomorphism counting.
To the best of our knowledge, EB-GNN has the strongest theoretical expressivity
guarantees among edge-based message-passing GNNs in the literature.
(2)~Unlike many GNN architectures that are more expressive than 1WL, we
prove that EB-1WL and EB-GNN admit near-linear time and memory usage on
practical graph learning workloads.
(3)~We show in experiments that EB-GNN is a highly efficient general-purpose
architecture: it substantially outperforms simple MPNNs and remains competitive
with task-specialized state-of-the-art GNNs at substantially lower computational
cost.
\end{abstract}

\section{Introduction}
\label{sec:intro}
Graph neural networks (GNNs) have emerged as a fundamental tool across the
sciences
\citep{DBLP:journals/tnn/ScarselliGTHM09,Kipf2016SemiSupervised,Wu2019ComprehensiveSurvey}. 
To explain the success and limitations of GNNs, a lot of research has been
dedicated to understanding the expressiveness of GNN architectures. While the
most prominent approach is the Weisfeiler--Leman test~\cite{xu18,MorrisAAAI19}, several alternative, yet equally insightful, characterizations of GNN
expressive power have been developed.  Particularly, researchers also obtained
characterizations of GNN expressiveness through finite-variable fragments of
counting logics
or homomorphism counts from classes
of graphs of bounded treewidth
\citep{DBLP:journals/jgt/Dvorak10,DBLP:conf/icalp/DellGR18}.

\begin{figure}
	\centering
\begin{tikzpicture}[every node/.style={circle, draw, thick, minimum size=1mm}]
	\def\s{2.5}
	
	\node (v) at (0,\s*1) {v};
	\node (u) at (0,\s*0) {u};
	\node (x) at (-\s*0.66,\s*1) {z};
	\node (y) at (-\s*0.66,\s*0) {x};
	\node (z) at (\s*1,\s*0.5) {y};
	
	\node[draw=none, fill=none] (w) at ($(u)!0.5!(v) + (0.18,0)$) {}; 
	\draw[yellow, bend right] (\s*0.6, 0.22) to ($(w) - (0.1,0) $);
	\draw[yellow, bend left] (\s*0.6, 1*\s - 0.22) to ($(w) - (0.1,0) $);
	
	\path let \p1 = ($(v)!0.5!(x)$) in coordinate (m_vx) at (\p1);
	\path let \p2 = ($(u)!0.5!(y)$) in coordinate (m_uy) at (\p2);
	
	\path let \p3 = ($(u)!0.5!(v)$) in coordinate (m_uv) at (\p3);
	
	\path let \p3 = ($(u)!0.3!(v)$) in coordinate (m_uv1) at (\p3);
	\path let \p3 = ($(u)!0.7!(v)$) in coordinate (m_uv3) at (\p3);
	
	\draw[-{Triangle[width=6pt,length=6pt]}, purple] (m_vx) to[out=270, in=180] ($(m_uv3) - (0.05,0)$);
	\draw[-{Triangle[width=6pt,length=6pt]}, blue] (m_uy) to[out=90, in=180]  ($(m_uv1) - (0.05,0)$);
	
	\draw[-{Triangle[width=6pt,length=6pt]}, blue] (0.9,0.01)  to[out=90, in=0]  ($(m_uv1) + (0.05,0)$);
	\draw[-{Triangle[width=6pt,length=6pt]}, purple] (0.9, 1*\s - 0.01) to[out=270, in=0] ($(m_uv3) + (0.05,0)$);
	
	\draw[-{Triangle[width=9pt,length=9pt]}, line width=3pt] (u) -- (v);  
	\draw[-{Triangle[width=6pt,length=6pt]}] (v) -- (x);      
	\draw[-{Triangle[width=6pt,length=6pt]}] (u) -- (y);  
	\draw[-{Triangle[width=6pt,length=6pt]}, bend right] (u) to (z);  
	\draw[-{Triangle[width=6pt,length=6pt]}, bend left] (v) to (z);

	\node[regular polygon, regular polygon sides=3, fill=yellow, draw=none,
	minimum size=2pt, rotate=90, scale=0.4] at ($(w)$) {};
	
\end{tikzpicture}
	\caption{Message passing scheme of EB-1WL and EB-GNN. Colors correspond to aggregation types.}
	\label{fig:message_passing}
\end{figure}

A key limitation of existing expressivity results is that 1WL, despite its computational efficiency, fails to capture important motifs such as triangles and other substructures~\citep{DBLP:conf/nips/Chen0VB20,arvind2020weisfeiler,DBLP:conf/iclr/LanzingerB24}. The higher-order WL hierarchy offers a principled way to overcome these limitations: properties invisible to 1WL are often detected by $k$WL for some $k>1$. However, even modest increases in expressivity quickly become computationally prohibitive—already 2WL-inspired GNNs require quadratic memory and cubic time in the number of nodes, making them impractical for large graphs~\citep{DBLP:conf/nips/MaronBSL19}. This leads to a central question: \emph{Can we go beyond 1WL while retaining near-linear cost on sparse graphs?}

Additionally, most of the aforementioned expressivity results have been derived
for GNNs which perform message passing on a graph's vertices. However, recently
it has been observed that in certain application domains, such as chemistry,
GNNs that pass messages along the \emph{edges} tend to outperform their
vertex-centric competitors~\citep{jian2025reaction,vadaddi2024graph,song2020communicative,ChemProp,doi:10.1021/acs.jcim.1c00975}. Thus,
another important challenge is as follows:
\emph{Can we obtain strong expressivity guarantees for GNNs with edge-based message passing?}

\textbf{Our contributions.}
We provide clear positive answers to both of these research questions. We introduce
a novel GNN architecture which performs edge-based message passing and which
explicitly takes into account triangles to achieve improved expressivity.  We
prove strong theoretical expressivity results, showing that our architecture
is more expressive than 1WL and providing a logical characterization through
first-order logic and homomorphism counts. We further show that our architecture requires only
near-linear time in practice.  We are not aware of any other edge-based
message-passing GNNs in the literature with such clear understanding of expressivity and strong theoretical guarantees.

More concretely, we introduce the \emph{edge-based} 1WL test (EB-1WL) --- a
color-refinement test that is provably more expressive than 1WL while still
being computationally efficient. EB-1WL updates edge colors through incidence
relations and triangle-induced interactions (see Fig.~\ref{fig:message_passing}
and Sec.~\ref{sec:maintec}), an edge-centric perspective inspired by the classic
triangle-counting algorithm of \citet{chiba1985arboricity}. This approach
retains much of 2WL’s relational strength while avoiding its combinatorial
blowup: each iteration runs in near-linear time on sparse graphs. Concretely,
each iteration takes \(O(\alpha m)\) time, where \(m\) is the number
of edges and \(\alpha\) is the graph’s arboricity.
Since $\alpha$ is typically small in real-world graphs,\footnote{
	In all 39 datasets considered by \citet{eppstein2010listing} the arboricity is at most 201 even though their largest graph has 3.7 million nodes and 16.5 million edges. On 32/39 of their datasets, the arboricity is less than 60.
	We note that \citet{eppstein2010listing} report the degeneracy, which is an upper bound on the arboricity.
} EB-1WL achieves near-linear performance in practice. 

We further prove three expressiveness results for EB-1WL:
\begin{enumerate}
	\item EB-1WL strictly extends the expressive power of 1WL and of NC-1WL by
	\citet{DBLP:journals/tmlr/0015YJ24}.
	\item On the logical side, EB-1WL admits a precise characterization in terms of clique-based finite-variable fragments with counting quantifiers.
	\item On the homomorphism-count side, 
	we show that EB-1WL is at least as powerful as counting homomorphisms from chordal graphs of treewidth two.
\end{enumerate}

Building on this foundation, we introduce EB-GNNs, an edge-based message-passing
architecture that exactly matches the expressive power of EB-1WL. In
experiments, EB-GNNs perform strongly across diverse settings: 
\begin{enumerate}
	\item On synthetic benchmarks, EB-GNNs capture structural patterns beyond
	triangle counts and, in some cases, rival higher-order tests.
	\item On molecular datasets, they outperform standard baselines and remain
	competitive with task-specialized GNNs while being significantly more
	computationally efficient.
	\item On large-scale graphs, they scale effectively and achieve
	state-of-the-art accuracy.
\end{enumerate}
Together, these results establish EB-GNNs as an
expressive, yet highly efficient general-purpose architecture.

\textbf{Related work.}
Surprisingly, \emph{edge-based GNN architectures} are rare in the literature.
Recent work by \citet{DBLP:journals/tmlr/0015YJ24} introduced the {\em
	neighbor-communication} 1WL (NC-1WL) test, extending 1WL by incorporating edge
information within the neighborhood of each node in the graph. They show that
NC-1WL is strictly more expressive than 1WL and remains efficient, making it an
attractive refinement from a practical standpoint. However, from a
theoretical perspective, its study remains incomplete: while NC-1WL has a
well-defined placement within the WL hierarchy, no alternative characterizations
(such as logical or homomorphism-count based) are known. Here, we show that our
proposed EB-1WL is more expressive than NC-1WL and that in experiments our
methods achieve better results.
Additionally, \citet{DBLP:journals/tnn/ZhangXTT20} propose an architecture based on
edge convolution. \citet{DBLP:journals/pami/CaiLWJ22} studied the
application of standard message-passing GNNs to the line graph\footnote{In the
	\emph{line graph} of a graph $G$, the edges of $G$ become vertices that are
	connected according to their incidences in $G$.} for link prediction.

Another line of related work concerns the study of \emph{efficient higher-order GNNs}
which is typically conducted for vertex-based message passing. Here various
approaches have been proposed to improve on higher-order $k$-WL performance
bottlenecks especially for sparse graphs.  \citet{DBLP:conf/nips/0001RM20}
introduced $\delta$-$k$-WL and its local variants, which demonstrate
improvements on sparse graphs while maintaining high expressivity.  However, in Appendix
\ref{app:relationship-delta-k-lwl} we argue that our approach has a
substantially faster running time and that the approaches differ significantly
from a technical perspective.
\citet{DBLP:conf/nips/ZhaoSA22} follow a similar motivation as
\citet{DBLP:conf/nips/0001RM20} and introduced $(k,c)(\leq)$-SetWL, which
approximates higher-order power in a fine-grained way by also embedding local
substructures.
While these refinements are more scalable than plain 2WL, the vertex-based paradigm they follow makes quadratic (in the number of vertices) running time unavoidable, even in sparse graphs.\footnote{The closest comparison to our method in terms of $(k,c)(\leq)$-SetWL is achieved with $k=3,c=1$ (lower $k$ yields expressivity at most 1WL). On a star graph this materializes all quadratically many 2-hop paths as nodes in a ``super-graph'' on which message-passing is performed. In contrast, a star has arboricity $1$ and our method is strictly linear in its running time.} 
Moreover, such practical considerations remove these methods from the strong theoretical foundations of $k$-WL, which admits well-known and highly influential characterizations in terms of homomorphism counts~\citep{DBLP:journals/jgt/Dvorak10,DBLP:conf/icalp/DellGR18} and variants of first-order logic~\citep{DBLP:journals/combinatorica/CaiFI92}.
Along the same lines, other prominent approaches to obtain efficient higher-order GNNs such as PPGN \citep{DBLP:conf/nips/MaronBSL19} also cannot avoid fundamentally quadratic (or higher) computational time complexity per layer.

\section{Preliminaries}
\label{sec:prelim}

\paragraph{Graphs.} We study undirected graphs $G=(V,E)$, where $V$ is the set of vertices and $E \subseteq \binom{V}{2}$ is the set of edges. 
We set $n=|V|$ and $m=|E|$. 
We write $N(v) := \{ w \mid \{v,w\} \in E\}$ for the \emph{neighborhood} of $v \in V$. For technical simplicity, we assume that graphs do not have isolated nodes.

The \emph{arboricity}~$\alpha$ of a graph $G=(V,E)$ is the minimum number of forests that partition its edge set~$E$~\cite{diestel2012graph}. More formally, the arboricity is the smallest integer of $k$ such that there exist forests $F_1,\dots,F_k$ with $F_i=(V,E_i)$ such that $\bigcup_{i=1}^k E_i = E$. The arboricity is also tightly related to other important graph parameters, such as the \emph{degeneracy} or the density of the \emph{densest subgraph}~\citep{nash1961edge}. It is widely known that real-world datasets such as social networks, road networks or planar graphs have very small arboricities~\cite{eppstein2010listing}.

A graph is \emph{chordal} if it has no induced subgraph that is a cycle of length 4 or larger. 
Note that, by definition, trees are always chordal.

\paragraph{WL test.} The Weisfeiler--Leman test (WL test) is a family of combinatorial 
algorithms for distinguishing graphs through iterative refinement 
of vertex- or tuple-colorings \citep{WeisfeilerLeman1968,DBLP:journals/combinatorica/CaiFI92}. 
The most widely used variant is the 
1WL test, or {\em color refinement}. Given a graph 
\(G = (V, E)\), this algorithm assigns each vertex \(v \in V\) 
a color \({\sf cr}^{(\ell)}(G,v)\) at iteration \(\ell \geq 0\), defined inductively 
as follows: the initial color is constant, \({\sf cr}^{(0)}(G,v) := 1\), and the update 
rule is
\[
{\sf cr}^{(\ell+1)}(G,v) := \big(\,{\sf cr}^{(\ell)}(G,v), \, \{\!\{ {\sf cr}^{(\ell)}(G,u) \mid u \in N(v)\}\!\} \, \big). 
\]
At each iteration $\ell$, the coloring induces a partition of the vertex set, where the partition at iteration $\ell+1$ refines that at iteration $\ell$. Once this process stabilizes, we write ${\sf cr}(G,v)$ for the final color assigned to vertex $v$. We define ${\sf cr}(G)$ as the multiset $\{\!\{ {\sf cr}(G,v) \mid v \in V\}\!\}$, and call two graphs $G$ and $G'$ {\em distinguishable by 1WL} if ${\sf cr}(G) \neq {\sf cr}(G')$.

\paragraph{The $k$WL test.}
Higher-order versions of the WL test are also widely studied in the literature. 
Rather than focusing on individual vertices, the $k$WL test, for $k > 1$, considers $k$-tuples of vertices. 
Each $k$-tuple is assigned a color that reflects the isomorphism type of the subgraph induced by those vertices. 
At each refinement step, the color of a $k$-tuple is updated by considering all possible ways of replacing one of its vertices with another vertex in the graph. 
For example, when $k=2$, each ordered pair $(u,v)$ refines its color by looking at pairs such as $(u,w)$ and $(w,v)$ for all possible $w$. In this way, the $k$WL test captures not only the view of individual vertices, but also the structural relations within patterns of size $k$.

\emph{Neighbor-communication WL test.} An extension of 1WL that  incorporates information about the 
edges between the neighbors of a node has been proposed recently 
\citep{DBLP:journals/tmlr/0015YJ24}. The resulting test, called 
\emph{NC-1WL test}, where NC stands for {\em neighbor communication}, proceeds similarly to 1WL, but updates the color ${\sf nc}^{(\ell)}(G,v)$
of each vertex \(v\) in a graph \(G = (V,E)\) according to the rule
\begin{multline*}
	{\sf nc}^{(\ell+1)}(G,v) \ := \ \big(\,{\sf nc}^{(\ell)}(G,v), \, 
	\{\!\{ {\sf nc}^{(\ell)}(G,u) \mid u \in N(v)\}\!\}, \\
	\{\!\{ ({\sf nc}^{(\ell)}(G,u), {\sf nc}^{(\ell)}(G,w)) \mid u,w \in N(v), \, \{u,w\} \in E \}\!\} 
	\big).
\end{multline*}

In other words, besides taking into account the multiset of colors of the 
neighbors of a vertex $v$, as in 1WL, the NC-1WL test also considers the 
multiset of color pairs corresponding to the edges within the neighborhood of $v$. 
We let ${\sf nc}(G,v)$ denote the color of node $v$ once the coloring partition on vertices defined by the NC-1WL test becomes stable. We write ${\sf nc}(G)$ for the multiset $\{\!\{ {\sf nc}(G,v) \mid v \in V\}\!\}$, and call two graphs $G$ and $G'$ {\em distinguishable by NC-1WL} if ${\sf nc}(G) \neq {\sf nc}(G')$.

The additional structural information collected by NC-1WL makes it strictly more powerful than 
1WL in distinguishing certain classes of non-isomorphic graphs. 
In turn, every pair of graphs that can be distinguished by 
NC-1WL can also be distinguished by 2WL, but the converse does not hold. 
The advantage of NC-1WL over 2WL, however, is that it achieves  stronger 
discriminative power while still operating at the \emph{vertex level}, in the 
same spirit as 1WL. In particular, its computational cost is closer to that of 
1WL than to the more demanding 2WL procedure.

\section{Edge-based WL test}
\label{sec:maintec}

In this section, we introduce the EB-1WL test (for edge-based 1WL), a color
refinement test that is more expressive than 1WL and NC-1WL and that colors
edges rather than vertices --- this is analogous to higher-order WL tests that
color vertex tuples. Unlike those tests, EB-1WL colors only edge pairs, reducing
space complexity from quadratic in nodes to linear in edges, making it more
practical. 

Let $G = (V,E)$ be an undirected graph. We iteratively associate a color ${\sf eb}^{(\ell)}(G,(u,v))$ with each ordered pair $(u,v)$ such that $\{u,v\} \in E$. That is, each edge receives two colors, one for each ordering of its endpoints. The coloring of the ordered pair $(u,v)$ is defined inductively. Initially, every pair has the same color: ${\sf eb}^{(0)}(G,(u,v)) = 1$. At iteration $\ell+1$, the color of $(u,v)$ is updated according to
\begin{align}
	\label{eq_old}
	&{\sf eb}^{(\ell+1)}(G,(u,v)) = \Big(\, {\sf eb}^{(\ell)}(G,(u,v)),\\
	\label{eq_u}
	&\Cline[blue]{\bleftm {\sf eb}^{(\ell)}(G,(u,x)) \mid x\in N(u)\brightm},\\
	\label{eq_common}
	&\Cline[yellow]{\bleftm \left({\sf eb}^{(\ell)}(G,(u,y)),{\sf eb}^{(\ell)}(G,(v,y))\right) \mid y \in N(v)\cap N(u) \brightm},\\
	\label{eq_v}
	&\Cline[purple]{\bleftm {\sf eb}^{(\ell)}(G,(v,z)) \mid z\in N(v)\brightm} \Big).
\end{align}
This update rule refines the color of an edge based on the colors of edges incident to its endpoints as can be seen in Fig.~\ref{fig:message_passing}. Eq.~\eqref{eq_u} captures the influence of edges incident to $u$, Eq.~\eqref{eq_v} does the same for edges incident to $v$, and Eq.~\eqref{eq_common} encodes the interaction between edges that are incident to both $u$ and~$v$, thereby forming a triangle and incorporating the local neighborhood structure around the edge. 

We let ${\sf eb}(G,(u,v))$ denote the color of the ordered pair $(u,v)$ once the coloring partition of the edges defined by the EB-1WL test becomes stable. We write ${\sf eb}(G)$ for the multiset $\{\!\{ {\sf eb}(G,(u,v)),{\sf eb}(G,(v,u)) \mid \{u,v\} \in E\}\!\}$ and call two graphs $G$ and $G'$ {\em distinguishable by EB-1WL} if they have a different number of vertices or 
${\sf eb}(G) \neq {\sf eb}(G')$.

\paragraph{Computational cost of EB-1WL.}\ \ 
Next, we study the computational cost of EB-1WL.
We consider an idealized computational model where a multiset of $s$~elements can be created in time $O(s)$ and stored in $O(1)$ space (in practice, this can be achieved with high probability using hashing). In this model, we need $O(m)$ space to store the graph and the colors of the edges.

Now we analyze the time needed to perform a single iteration. 
First, computing the multisets in Eq.s~\eqref{eq_u} and~\eqref{eq_v} for all nodes~$u$ and~$v$ can be done in total time \(O(m)\), since for every node~$u$ (and, resp.,~$v$) we spend time proportional to its degree.  
The more interesting part is computing Eq.~\eqref{eq_common} efficiently for all \emph{ordered} edges \((u,v)\). 
For this, we need to go through all common neighbors~$y$ of~$u$ and~$v$. A na\"ive approach would iterate over all neighbors~$y$ of~$u$ and then check whether \((y,v)\) exists. This becomes slow if~$u$ has a much higher degree than~$v$, potentially taking total time \(O(md)\), where~$d$ is the maximum degree of the graph.  

	Instead, following the triangle-enumeration algorithm of \citet{chiba1985arboricity}, we iterate only over the neighbors of the lower-degree endpoint of $(u,v)$ and check adjacency to the other endpoint. This still visits every common neighbor $y$ of $u$ and $v$  exactly once, so every triangle $(u,v,y)$ is included in the summation in Eq.~\eqref{eq_common}; only the order and implementation of the enumeration change, and the permutation-invariant aggregation therefore remains unaffected. The analysis of \citet{chiba1985arboricity} then implies a running time of $O(\alpha m)$, where $\alpha$ is the arboricity of the graph. We thus immediately get:

\begin{proposition}
	\label{prop_time}
	An iteration of EB-1WL 
	can be performed using $O(m)$ space and $O(\alpha m)$ time.
\end{proposition}

 Appendix \ref{sec:relationship} provides more details on this algorithm and its connection to EB-1WL.
We note that this running time is highly efficient in practice and that summing
over all triangles polynomially faster than in $O(\alpha m)$ time would violate established
conjectures from the computational complexity
community~\citep{kopelowitz2016higher,williams2020monochromatic}.

\section{The distinguishing power of EB-1WL}
\label{sec:distinguishing} 
We present our theoretical results on the expressiveness of EB-1WL,
ranging from a comparison with 1WL and NC-1WL (Sec.~\ref{sec:comparison-nc}),
over a logical characterization (Sec.~\ref{sec:logical-characterization}) to its
power based on homomorphism counts (Sec.~\ref{sec:homomorphism-characterization}).

\subsection{Expressivity}
\label{sec:comparison-nc}

We start by showing that the EB-1WL test is strictly more expressive than 1WL
and the edge-based NC-1WL test by \citet{DBLP:journals/tmlr/0015YJ24} in
distinguishing non-isomorphic graphs.

	\begin{theorem} \label{theo:versus}
		Every pair of graphs distinguishable by NC-1WL is also distinguishable by EB-1WL. Additionally, there are graphs $G$ and $H$ that 
		can be distinguished by EB-1WL but not by NC-1WL. Thus, EB-1WL is strictly more expressive than NC-1WL and 1WL.
	\end{theorem}
	\begin{proof}
		Our \emph{proof idea} is as follows. Looking at the graphs in
		Fig.~\ref{fig:example}, all nodes of $G$ and $H$ have the same degree and
		lie in the same number of triangles. Hence, they are not distinguished by
		NC-1WL.  However, they are distinguished by EB-1WL since every edge in $H$ is
		part of two triangles, whereas $G$ has edges that belong to only a single
		triangle.	Next, we give the formal proof.
		
		\paragraph{Part 1:} 
		We first show the claim that every pair of graphs distinguishably by NC-1WL is also distinguished by EB-1WL.
		
		Fix a graph $G = (V,E)$ (we omit $G$ in the notation for colors from now on).   In the proof, we use the following terminology. An ``ordered edge'' is an ordered pair of nodes, connected by an edge (EB-1WL assigns colors to ordered edges). Now, an ``ordered triangle'' is an ordered triple of nodes where all nodes are connected by an edge.

		To simplify the presentation, we will also use the following notation. For example, if $u$ is a node of $G$, then $(u,*)$ is the set of ordered pairs of nodes connected by an edge, where the first node in the pair is $u$. Likewise, $(*,u)$ will be a similar set but for pairs where the second node is $u$.
		
		More generally, if we have a $k$-tuple where some coordinates are nodes of $G$, and some coordinates are $*$ (meaning ``undefined'), this tuple denotes the set of all ways to replace $*$'s by nodes of $G$ such that all pairs of nodes in the tuple are connected by an edge. For example, $(*,*, *)$ means the set of all ordered triangles, and $(u,*,*)$ denotes the set of all ordered triangles that have $u$ as the first coordinate.
		
		Next, if $c$ is a coloring of nodes, and $t$ is a tuple of nodes, we will write $c(t)$ for the tuple of colors of nodes in $t$, for instance, if $t = (u,v,w)$, then $c(t)  = (c(u), c(v), c(w))$. Likewise, if $c$ is a coloring of pairs of nodes, and $t$ is a tuple, then we will write $c(t)$ for the tuple of colors of all pairs of nodes in $t$. We will use this when $t$ is at most a triple of nodes, where then it is defined as:
		\[c(u,v,w) = (c(u,v), c(u,w), c(v,w)).\]
		
		Moreover, we extend this notation to tuples with $*$s. Namely, if $T$ is a tuple with stars (a set of tuples of nodes without stars), then $c(T) = \bleftm c(t) \mid t\in T\brightm$.

		In this notation, the updates of NC-1WL and EB-1WL can be defined as follows. The color $nc^{(\ell+1)}(u)$ is defined by multisets $nc^{(\ell+1)}(u,*)$ and $nc^{(\ell+1)}(u,*,*)$. In turn, the color $\eb^{(\ell+1)}(u,v)$ is defined by  $\eb^{(\ell)}(u,*)$, $\eb^{(\ell)}(v,*)$, and $\eb^{(\ell)}(u,v,*)$.

		\bigskip
		
		To establish the claim, it is enough to show that $\eb(*,*)$ uniquely determines $\nc^{(\ell)}(*)$ for every $\ell \ge 0$.
		\begin{lemma}
			For every $\ell \ge 0$, and for  every ordered edge $(u,v)$, we have that $\eb(u,v)$ uniquely determines $\nc^{(\ell)}(u)$ and  $\nc^{(\ell)}(v)$. 
		\end{lemma}
		\begin{proof}
			The proof is by induction on $\ell$. For $\ell = 0$, all nodes have the same $\nc^{(0)}$-color, so there is nothing to prove. Assume now that the statement is proved for $\ell$, we establish it for $\ell + 1$. That is, we have to show that  $\eb(u,v)$ uniquely determines $\nc^{(\ell+1)}(u)$ and $\nc^{(\ell+1)}(v)$. We only show that it determines $\nc^{(\ell+1)}(u)$. It is enough  because by induction in (\ref{eq_old}--\ref{eq_v}), it can be shown that $\eb(u,v)$ uniquely determines $\eb(v,u)$.
			
			By definition, $\eb^{(\ell)}(u,v)$ uniquely determines $\eb^{(\ell-1)}(u,*)$ and  $\eb^{(\ell-1)}(u,v,*)$. If we make one more step, from $\eb^{(\ell-1)}(u,*)$ we can determine $\eb^{(\ell-2)}(u,*,*)$.  Indeed, we use the fact that we can determine $\eb^{(\ell-2)}(u,w,*)$ from $\eb^{(\ell-1)}(u,w)$ for $(u,w) \in (u,*)$.

			For a stable EB-1WL coloring $\eb$, we thus get that $\eb(u,v)$ uniquely determines $\eb(u,*)$ and $\eb(u,*,*)$. They, in turn, by the induction hypothesis, uniquely determined $\nc^{\ell}(u,*)$ and $\nc^{\ell}(u,*,*)$. Thus, from this information, we get  $\nc^{(\ell+1)}(u)$, as required. 
		\end{proof}
		We now finish the proof of the claim. Assume that we want to know how many times a node color $c$ appears in $\nc^{(\ell)}(*)$. We can assume that $\ell \ge 1$ (we can determine the multiset for $\ell = 0$ from the multiset for $\ell = 1$), then the color $c$ uniquely determines the degree $d$ of a node (which is not 0 by the assumption of the absence of isolated nodes).
		We go through all ordered edges $(u,v)$ and count how many times we have $\nc^{(\ell)}(u) = c$. We count every node $d$ times, so we have to divide this number by $d$. 
		
		\paragraph{Part 2:} We move on to the second part of the theorem: that there are graphs that are distinguished by EB-1WL but not by NC-1WL. 
		The graphs $G$ and $H$ are given in Fig.~\ref{fig:example}. Both have 16 nodes and are based on the 16-gons. The graph $G$ has also all chords for vertices at distance 2 and at distance 4 in the 16-gon. In turn, the graph $H$ has all chords for vertices at distance 3 and distance 4 in the 16-gon.
		
		\begin{figure}[t]
			\centering
			\scalebox{1.1}{
				\subfloat[\centering graph $G$]{\begin{tikzpicture}
						\node[draw=none,minimum size=3cm,regular polygon,regular polygon sides=16] (a) {};
						
						\foreach \x in {0,1,...,15}
						\node (\x) at ({3 * cos((360 * \x)/16)}, {3 *sin((360 * \x)/16)}) {\textbullet};
						
						\foreach \x in {0,1,...,15}
						\draw[blue, thick] ({3 *cos((360 * \x)/16)}, {3 *sin((360 * \x)/16)}) -- ({3 * cos((360 * (\x+1))/16)}, {3 * sin((360 * (\x+1))/16)});
						
						\foreach \x in {0,1,...,15}
						\draw[red, thick] ({3 *cos((360 * \x)/16)}, {3 *sin((360 * \x)/16)}) -- ({3 * cos((360 * (\x+2))/16)}, {3 * sin((360 * (\x+2))/16)});
						
						\foreach \x in {0,1,...,15}
						\draw[green, thick] ({3 *cos((360 * \x)/16)}, {3 *sin((360 * \x)/16)}) -- ({3 * cos((360 * (\x+4))/16)}, {3 * sin((360 * (\x+4))/16)});

				\end{tikzpicture}}
				\qquad
				\subfloat[\centering graph $H$]{\begin{tikzpicture}
						\node[draw=none,minimum size=3cm,regular polygon,regular polygon sides=16] (a) {};
						
						\foreach \x in {0,1,...,15}
						\node (\x) at ({3 * cos((360 * \x)/16)}, {3 *sin((360 * \x)/16)}) {\textbullet};
						
						\foreach \x in {0,1,...,15}
						\draw[blue, thick] ({3 *cos((360 * \x)/16)}, {3 *sin((360 * \x)/16)}) -- ({3 * cos((360 * (\x+1))/16)}, {3 * sin((360 * (\x+1))/16)});
						
						\foreach \x in {0,1,...,15}
						\draw[red, thick] ({3 *cos((360 * \x)/16)}, {3 *sin((360 * \x)/16)}) -- ({3 * cos((360 * (\x+3))/16)}, {3 * sin((360 * (\x+3))/16)});
						
						\foreach \x in {0,1,...,15}
						\draw[green, thick] ({3 *cos((360 * \x)/16)}, {3 *sin((360 * \x)/16)}) -- ({3 * cos((360 * (\x+4))/16)}, {3 * sin((360 * (\x+4))/16)});

			\end{tikzpicture}}}
			\caption{Example}
			\label{fig:example}
		\end{figure}

		All nodes in both graphs have degree 6. Moreover, all nodes in both graphs appear in the same number of triangles, namely, 6. Indeed, every triangle in both graphs is given by an equation $a + b = c$, for some chords, connecting vertices at distances $a,b$ and $c$, respectively. In the first graph, there are two types of triangles, $1+ 1 = 2$ and $2 + 2 = 4$. Each type, when we rotate it over the 16-gon, covers each vertex exactly 3 times, giving 6  triangles for each node. In the second graph, there are also two types of triangles, $1 + 3 = 4$ and $3 + 1 = 4$ (differing just by their orientation), so the same count gives 6  triangles for every node.
		
		This implies that the NC-1WL gives the same color to all nodes of both graphs in the first iteration, and thus, cannot distinguish these two graphs.

		On the other hand, these two graphs can be distinguished by the EB-1WL. Indeed, in $H$, every edge appears in 2 triangles (every 1-edge or every 3-edge we can complement either from the left or from the right, and every 4-edge we can complete in two ways inside). But, say, for the 4-edge in the first graph, there is just one 2 + 2 triangle that has it.
	\end{proof}

Interestingly, in \Cref{app:wl-triangle} we also show that EB-1WL is strictly
more expressive than 1WL with triangle counts added on node and edge level.
Furthermore, we note that any pair of graphs distinguishable by EB-1WL is also
distinguishable by 2WL, though the reverse implication does not hold. A concrete
example of a pair of graphs that is distinguishable by 2WL but neither by EB-1WL
nor NC-1WL is shown in Fig.~\ref{fig:EB-1WL}.

\begin{figure} 
	\centering
	\begin{tikzpicture}[every node/.style={circle, draw, thick, minimum size=5mm}]
		
		\node (a) at (0,1) {};
		\node (b) at (1,1) {};
		\node (c) at (2,1) {};
		\node (g) at (3,1) {};
		\node (d) at (0,0) {};
		\node (e) at (1,0) {};
		\node (f) at (2,0) {};
		\node (h) at (3,0) {};
		
		\foreach \u/\v in {a/b, b/c, c/g, d/e, e/f, f/h, a/d, g/h}
		\draw (\u) -- (\v);
		
		\node[draw=none, fill=none, shape=rectangle, inner sep=0, anchor=west]
		at ([xshift=3mm]current bounding box.east) {$G_1$};
	\end{tikzpicture}

	\bigskip
	
	\begin{tikzpicture}[every node/.style={circle, draw, thick, minimum size=5mm}]
		\node (a2) at (6,0) {};
		\node (b2) at (6,1) {};
		\node (c2) at (7,0) {};
		\node (g2) at (7,1) {};
		\node (d2) at (8,0) {};
		\node (e2) at (9,1) {};
		\node (f2) at (9,0) {};
		\node (h2) at (8,1) {};
		
		\foreach \u/\v in {a2/b2, b2/g2, c2/a2, c2/g2, d2/f2, d2/h2, e2/f2, h2/e2}
		\draw (\u) -- (\v);

		\node[draw=none, fill=none, shape=rectangle, inner sep=0, anchor=west]
		at ([xshift=3mm]current bounding box.east) {$G_2$};
	\end{tikzpicture}
	\caption{Graphs distinguishable by 2WL but not by EB-1WL.}
	\label{fig:EB-1WL}
\end{figure}  

\subsection{Logical characterization} 
\label{sec:logical-characterization}

Next, we provide a complete characterization of EB-1WL's expressiveness through
first-order logic. To highlight the merit of this result, we note that the distinguishing power of the $k$-WL test can be characterized via a fragment of first-order logic, namely the \emph{$(k+1)$-variable fragment with counting quantifiers} \citep{DBLP:journals/combinatorica/CaiFI92}, with 1WL and 2WL corresponding to the two- and three-variable fragments, respectively. 
These logical characterizations have been an important tool in studying GNNs, as they offer insights into what GNNs can accomplish (cf.,~\cite{DBLP:conf/iclr/BarceloKM0RS20}).

We show that EB-1WL also admits a natural logical characterization, further demonstrating the robustness and naturalness of our newly proposed framework.
Specifically, EB-1WL corresponds to a novel natural logic we introduce below.
This logic lies strictly between the two- and three-variable counting fragments,
providing a precise characterization of its position between 1WL and 2WL.

In the following, we assume familiarity with the syntax and semantics of
first-order logic (FO). We let $\phi$ be a formula and write $\phi(\bar x)$ to indicate that the free variables of $\phi$ are exactly those in the tuple $\bar x$.
When interpreted over graphs, FO formulas are defined over a vocabulary consisting of a single binary relation symbol $E$. Specifically, the formula $E(x,y)$ is interpreted over a graph $G = (V,E)$ as the set of all pairs $(u,v) \in V \times V$ such that $\{u,v\} \in E$.
Given a tuple $\bar x = (x_1, \dots, x_n)$ of distinct variables, we define the FO formula
\[
{\sf clique}(\bar x) \ := \ \bigwedge_{1 \leq i < j \leq n} E(x_i, x_j),
\]
so that its interpretation over $G$ is the set of all $n$-tuples of vertices that form a clique in $G$.

We now introduce our novel \emph{clique-based first-order logic with counting (CFOC)} via the following syntax:

\begin{enumerate}
	\item If $\bar x$ is a tuple of distinct variables, then ${\sf clique}(\bar x)$ is a formula in CFOC.
	\item If $\phi(\bar x)$ is a formula in CFOC, then so is ${\sf clique}(\bar x) \wedge \neg \phi(\bar x)$.
	\item If $\phi(\bar x)$ and $\psi(\bar y)$ are formulas in CFOC, then so is
	${\sf clique}(\bar z) \wedge (\phi(\bar x) \star \psi(\bar y))$ for any $\star \in \{\wedge,\vee\}$, where $\bar z$ is the tuple obtained by collecting all variables that occur in 
	$\bar{x}$ and $\bar{y}$, removing any duplicates, so that each variable appears exactly once.
	\item If $\phi(\bar x,y)$ is a formula in CFOC, then 
	${\sf clique}(\bar x) \wedge \exists^{\ge k} y \, \phi(\bar x, y)$
	is a formula in CFOC, for any integer $k \ge 1$, where the semantics of $\exists^{\ge k} y \, \phi(\bar x, y)$ assert that there exist at least $k$ vertices~$v$ such that $\phi(\bar x, y)$ holds when $y=v$.
\end{enumerate}
Intuitively, CFOC restricts FO with counting quantifiers to formulas whose free variables form a clique in the graph. We note that this is reminiscent of, but not directly related to, the so-called {\em clique-guarded fragments} of first-order logic; see, e.g., \citet{DBLP:conf/cade/Gradel99}.

We denote CFOC$^3$ the fragment of CFOC that consists of formulas that use at most three variables. For example, the CFOC$^3$ formula 
$\psi := \exists^{\ge 1} x\exists^{\ge 1} y \ \big( \, {\sf clique}(x,y) \land \exists^{\geq 3} z \,{\sf clique}(x,y,z) \, \big)$ 
checks if there is an edge which is a part of at least 3 triangles.  
A CFOC$^3$ {\em sentence} is a CFOC$^3$ formula without free variables. We call
graphs $G$ and $H$ {\em distinguishable by CFOC$^3$}, if there is a CFOC$^3$
sentence $\phi$ such that $G \models \phi$ but $H \not\models \phi$. We can now
establish our characterization of EB-1WL.

	\begin{example} 
	The sentence $\psi$ shown above holds in graph $H$ from Figure
	\ref{fig:example}, but not in graph $G$. Thus, EB-1WL is able to
	distinguish these graphs.
	\hfill  $\blacktriangle$
\end{example} 	
	
	\begin{theorem}
		\label{thm_logic}
		The pairs of graphs that are distinguishable by EB-1WL are precisely those that are distinguishable by ${\rm CFOC}^3$.
	\end{theorem}
	\begin{proof}
		   We start by showing the following. 
		
		\begin{lemma} \label{lem:types1} 
			For each possible color $c$ that the EB-1WL test obtains after $t$ rounds, there is a formula $\phi_c(x,y)$ of quantifier depth $t$ in CFOC$^3$ such that ${\sf eb}^t(G,(u,v)) = c$ if and only if $G \models \phi_c(u,v)$, for each graph $G = (V,E)$ and edge $\{u,v\} \in E$. 
		\end{lemma}
		
		\begin{proof} 
			We prove the result by induction on the number of rounds $t \geq 0$.  
			By definition of EB-1WL, before any refinement every edge $\{u,v\}$ receives the same initial color $c=1$. Accordingly, we can define 
			$\phi_c(x,y) \ := \ E(x,y)$.   
			
			Suppose we already have, for every color $d$ occurring after $t$ rounds of EB-1WL, a formula $\phi_d(x,y)$ that defines it. 
			Consider now a new color $c$ obtained after $t+1$ rounds.  
			By the definition of EB-1WL, the color of an edge $\{u,v\}$ at step $t+1$ is completely determined by the following information from step $t$:  
			\begin{enumerate}
				\item the color $d$ previously assigned to $\{u,v\}$,  
				\item the multiset $\mathcal{E}$ of colors of edges of the form $\{u,w\}$,  
				\item the multiset $\mathcal{G}$ of colors of edges of the form $\{v,w\}$, and  
				\item the multiset $\mathcal{F}$ of pairs of colors $(f_1,f_2)$ corresponding to edges $\{u,w\}$ and $\{v,w\}$ incident with a common neighbor $w$.  
			\end{enumerate}
			
			It follows that the new color $c$ can be described in CFOC$^3$ by the formula: 
			\begin{multline*} 
				\phi_c(x,y) \ := \ \phi_d(x,y) \, \wedge \\ 
				\bigwedge_{e \in {\cal E}} \exists^{= n_e} z \big(E(x,z) \wedge \phi_e(x,z)\big) \, \wedge \,  
				\exists^{= \sum_{e \in {\cal E}} n_e} z E(x,z) \\ 
				\bigwedge_{g \in {\cal G}} \exists^{= m_g} z \big(E(z,y) \wedge \phi_g(z,y)\big) \, \wedge \, \exists^{= \sum_{g \in {\cal G}} m_g} z E(y,z) \\ 
				\bigwedge_{(f_1,f_2) \in {\cal F}} \exists^{= p_{(f_1,f_2)}} z \big(E(x,z) \wedge E(y,z) \wedge \phi_{f_1}(x,z) \wedge \phi_{f_2}(y,z)\big) \, \wedge \\ \exists^{= \sum_{(f_1,f_2) \in {\cal F}} p_{(f_1,f_2)}} z \big(E(x,z) \wedge E(y,z)\big).
			\end{multline*} 
			Here, $n_e$ is the multiplicity of $e$ in $\mathcal{E}$, $m_g$ the multiplicity of $g$ in $\mathcal{G}$, and $p_{(f_1,f_2)}$ the multiplicity of $(f_1,f_2)$ in $\mathcal{F}$. 
			As usual, we write $\exists^{=n} x \, \psi$ as shorthand for $\exists^{\geq n} x \, \psi \ \wedge \ \neg \exists^{\geq n+1} x \, \psi$. Notice that $\phi_c(x,y)$ has quantifier rank $t+1$.  
		\end{proof}     
		
		Let $G = (V,E)$ be a graph and $\bar v$ a tuple of elements in $V$. For $t \geq 0$, we write ${\sf tp}^t(G,\bar v)$ for the {\em $t$-type of $(G,\bar v)$}, which is the set of formulas $\phi(\bar x)$ with quantifier depth $t$ in CFOC$^3$ such that $G \models \phi(\bar v)$.
		Also, ${\sf eb}^t(G)$ is the multiset formed by all elements of the form ${\sf eb}^t(G,(u,v))$, where $\{u,v\} \in E$. 
		
		\begin{lemma} \label{lem:types} 
			Consider graphs $G = (V,E)$ and $G' = (V',E')$ and edges $\{u,v\} \in E$ and $\{u',v'\} \in E'$. Then for each $t \geq 0$, the following hold: 
			\begin{enumerate} 
				\item If ${\sf eb}^t(G) = {\sf eb}^t(G')$ then ${\sf tp}^t(G) = {\sf tp}^t(G')$. 
				\item If ${\sf eb}^t(G,(u,v)) = {\sf eb}^t(G',(u',v'))$ and ${\sf eb}^t(G) = {\sf eb}^t(G')$, then ${\sf tp}^t(G,u) = {\sf tp}^t(G',u')$, ${\sf tp}^t(G,v) = {\sf tp}^t(G',v')$, and ${\sf tp}^t(G,(u,v)) = {\sf tp}^t(G',(u',v'))$. 
			\end{enumerate}
		\end{lemma}
		
		\begin{proof} 
			We prove this by induction on $t \geq 0$. The base case $t = 0$ holds trivially. In fact, there are no sentences in CFOC$^3$ of quantifier depth $0$, and hence ${\sf tp}^0(G) = {\sf tp}^0(G')$ holds vacuously. Moreover, any formula $\phi(x)$ of quantifier depth $0$ with only one free variable $x$ in CFOC$^3$ is a Boolean combination of formulas of the form $E(x,x)$. Since both $G \models \neg E(u,u)$ and $G' \models \neg E(u',u')$, it is the case that $G \models \phi(u) \Leftrightarrow G' \models \phi(u')$. The same holds for $v$.  Finally, any formula $\phi(x,y)$ of quantifier depth $0$ with only two free variables $x$ and $y$ in CFOC$^3$ is a Boolean combination of formulas of the form $E(x,y)$, $E(x,x)$, and $E(y,y)$. Since both $G \models \neg E(u,u)$ and $G' \models \neg E(u',u')$, both $G \models \neg E(v,v)$ and $G' \models \neg E(v',v')$, and both $G \models E(u,v)$ and $G' \models \neg E(u',v')$,  it is the case that $G \models \phi(u,v) \Leftrightarrow G' \models \phi(u',v')$.
			
			Let us consider now the inductive case $t + 1$, for $t \geq 0$. We prove (1) and (2)  separately. 
			
			\begin{itemize}
				
				\item We first prove (1).  
				Assume that ${\sf eb}^{t+1}(G) = {\sf eb}^{t+1}(G')$.  
				Consider a sentence $\phi$ in CFOC$^3$ of quantifier depth $t+1$.  
				First, note that the case $t = 0$ is trivial: in this situation, $\phi$ must be a Boolean combination of formulas of the form $\exists^{\geq n} x\, \psi(x)$, where $\psi(x)$ is itself a Boolean combination of formulas of the form $E(x,x)$. Consequently, either every graph $G$ satisfies $\phi$ or none does.  
				
				Now assume $t > 0$, and let $\phi$ be a Boolean combination of CFOC$^3$ formulas of the form $\exists^{\geq n} x\, \phi(x)$, where $\phi(x)$ has quantifier depth $t$.  
				Our assumption ${\sf eb}^{t+1}(G) = {\sf eb}^{t+1}(G')$ implies  
				\[
				\{\!\{{\sf eb}^t(G,(u,v)) \mid \{u,v\}\in E\}\!\}
				= 
				\{\!\{{\sf eb}^t(G',(u',v')) \mid \{u',v'\}\in E'\}\!\},
				\]  
				and therefore, by the induction hypothesis,  
				\[
				\{\!\{{\sf tp}^t(G,(u,v)) \mid \{u,v\}\in E\}\!\}
				= 
				\{\!\{{\sf tp}^t(G',(u',v')) \mid \{u',v'\}\in E'\}\!\}.
				\]  
				Moreover, ${\sf tp}^t(G,(u,v))$ determines both ${\sf tp}^t(G,u)$ and ${\sf tp}^t(G,v)$.  
				
				To establish that $G \models \phi \iff G' \models \phi$, it suffices to show  
				\[
				\{\!\{{\sf tp}^t(G,u) \mid u \in V\}\!\}
				=
				\{\!\{{\sf tp}^t(G',u') \mid u' \in V'\}\!\}.
				\]  
				Take an arbitrary $\tau = {\sf tp}^t(G,u)$ for some $u \in V$.  
				Since $t>0$, every vertex $v \in V$ with ${\sf tp}^t(G,v)=\tau$ must have the same degree $d$ as $u$ (because they satisfy the same formulas of the form $\exists^{\geq n} y\, E(x,y)$). Define  
				\[
				S(G,\tau)
				= 
				\{\!\{{\sf tp}^t(G,(u,v)) \mid \{u,v\}\in E \text{ and } {\sf tp}^t(G,u)=\tau\}\!\}.
				\]  
				Then the number of vertices $v \in V$ and $v' \in V'$ such that ${\sf tp}^t(G,v) = \tau = {\sf tp}^t(G',v')$ is  
				\[
				|S(G,\tau)|/d 
				\quad \text{and} \quad 
				|S(G',\tau)|/d,
				\]  
				respectively.  
				Because $S(G,\tau) = S(G',\tau)$ by the above, the two counts coincide, which completes the proof.

				\item We now prove (2). Assume that ${\sf eb}^{t+1}(G,(u,v)) = {\sf eb}^{t+1}(G',(u',v'))$ and ${\sf eb}^{t+1}(G) = {\sf eb}^{t+1}(G')$. We only show that ${\sf tp}^{t+1}(G,(u,v)) = {\sf tp}^{t+1}(G',(u',v'))$, as the remaining cases are conceptually analogous. Consider a formula $\phi(x,y)$ in CFOC$^3$ of quantifier depth $t+1$. Then $\phi(x,y)$ must be of the form $E(x,y) \wedge \alpha(x,y)$, where $\alpha(x,y)$ is a CFOC$^3$ formula of quantifier depth $t+1$. Since both $G \models E(u,v)$ and $G' \models E(u',v')$, we only have to show that $G \models \alpha(u,v) \Leftrightarrow G \models \alpha(u',v')$. By definition, $\alpha(x,y)$ is a Boolean combination of: (a) CFOC$^3$ formulas $\beta(x,y)$ of quantifier depth $t+1$ with two free variables, $x$ and $y$, (b) CFOC$^3$ formulas $\gamma(x)$ of quantifier depth $t+1$ with only one free variable, $x$, (c) CFOC$^3$ formulas $\delta(y)$ of quantifier depth $t+1$ with only one free variable, $y$, and (d) CFOC$^3$ sentences $\eta$ of quantifier depth $t+1$. 
				
				Consider first a formula as above of the form $\beta(x,y)$. By construction, $\beta(x,y) = E(x,y) \wedge \exists^{\geq n} z\, (E(x,z) \wedge E(y,z) \wedge \beta_1(x,y,z))$, where $\beta_1(x,y,z)$ is a CFOC$^3$ formula of quantifier depth $t$. But since the logic only allows formulas to mention three variables, $\beta_1$ must be a Boolean combination of CFOC$^3$ formulas of quantifier depth $t$ with at most two free variables in the set $\{x,y,z\}$. By assumption, ${\sf eb}^{t+1}(G,(u,v)) = {\sf eb}^{t+1}(G',(u',v'))$, and hence 
				both ${\sf eb}^{t}(G,(u,v)) = {\sf eb}^{t}(G',(u',v'))$ and 
				\begin{multline*}
					\{\!\{\big({\sf eb}^t(G,(u,w)),{\sf eb}^t(G,(v,w))\big) \mid \{u,w\},\{v,w\} \in E\}\!\} \ = \\ 
					\{\!\{\big({\sf eb}^t(G',(u',w')),{\sf eb}^t(G',(v',w))\big) \mid \{u',w'\},\{v',w'\} \in E'\}\!\}. 
				\end{multline*} 
				Also by assumption, ${\sf eb}^{t+1}(G) = {\sf eb}^{t+1}(G')$, which implies that ${\sf eb}^{t}(G) = {\sf eb}^{t}(G')$. By induction hypothesis on (1) and (2), we conclude that ${\sf tp}^t(G) = {\sf tp}^t(G')$, ${\sf tp}^{t}(G,(u,v)) = {\sf tp}^{t}(G',(u',v'))$,  ${\sf tp}^{t}(G,u) = {\sf tp}^t(G',u')$, ${\sf tp}^{t}(G,v) = {\sf tp}^t(G',v')$, and: 
				\begin{multline*}
					\{\!\{\big({\sf tp}^t(G,(u,w)),{\sf tp}^t(G,(v,w)),{\sf tp}^t(G,w)\big) \mid \{u,w\},\{v,w\} \in E\}\!\} \ = \\ 
					\{\!\{(\big({\sf tp}^t(G',(u',w')),{\sf tp}^t(G',(v',w)),{\sf tp}^t(G',w')\big) \mid \{u',w'\},\{v',w'\} \in E'\}\!\}. 
				\end{multline*} 
				This implies that 
\[					|\{w \in N(u) \cap N(v) \mid G \models \beta_1(u,v,w)\}|  =  
					|\{w' \in N(u') \cap N(v') \mid G' \models \beta_1(u',v',w')\}|, 
\]
				which means that $G \models \beta(u,v) \Leftrightarrow G' \models \beta(u',v')$. 
				
				Consider second a formula as above of the form $\gamma(x)$. By construction, $\gamma(x) = \exists^{\geq n} z\, (E(x,z) 
				\wedge \gamma_1(x,z))$, where $\gamma_1(x,z)$ is a CFOC$^3$ formula of quantifier depth $t$. By assumption, ${\sf eb}^{t+1}(G,(u,v)) = {\sf eb}^{t+1}(G',(u',v'))$, and hence 
\[
					\{\!\{{\sf eb}^t(G,(u,w)) \mid \{u,w\} \in E\}\!\}  =  
					\{\!\{{\sf eb}^t(G',(u',w')) \mid \{u',w'\} \in E'\}\!\}. 
\]
				By induction hypothesis, this implies that: 
			\[
					\{\!\{{\sf tp}^t(G,(u,w)) \mid \{u,w\} \in E\}\!\}  =  
					\{\!\{{\sf tp}^t(G',(u',w')) \mid \{u',w'\} \in E'\}\!\}. 
\]
				By focusing on those types that contain $\gamma_1(x,z)$, we obtain: 
\[			
					|\{w \in N(u) \mid G \models \gamma_1(u,w)\}|  =  
					|\{w' \in N(u') \mid G' \models \gamma_1(u',w')\}|,\] 
				which means that $G \models \gamma(u) \Leftrightarrow G' \models \gamma(u')$. 
						The formulas of the form $\delta(y)$ are handled analogously to the previous case. 
				Finally, consider a sentence as above of the form $\eta$. Since ${\sf eb}^{t+1}(G) = {\sf eb}^{t+1}(G')$, we have by part (1) that $G \models \eta \Leftrightarrow G' \models \eta$. 
			\end{itemize} 
			This finishes the proof of the lemma. 
		\end{proof} 
		
		We now prove Theorem \ref{thm_logic}. Assume first that ${\sf eb}(G) = {\sf eb}(G')$, for graphs $G$ and $G'$. In particular, then, ${\sf eb}^t(G) = {\sf eb}^t(G')$ for every $t \geq 0$. Take an arbitrary sentence $\phi$ of CFOC$^3$, and assume that its quantifier depth is $t$. From Lemma \ref{lem:types}, we conclude that ${\sf tp}^t(G) = {\sf tp}^t(G')$, and hence $G \models \phi \Leftrightarrow G \models \phi'$. 
		
		Assume, on the contrary, that $G \models \phi \Leftrightarrow G \models \phi'$, for every CFOC$^3$ sentence $\phi$. Hence, ${\sf tp}^t(G) = {\sf tp}^t(G')$ for every $t \geq 0$. To prove ${\sf eb}(G) = {\sf eb}(G')$ it suffices to show ${\sf eb}^t(G) = {\sf eb}^t(G')$ for each $t \geq 0$. For a node $v \in V$, define ${\sf eb}^t(v) = \{\!\{{\sf eb}^t(v,w) \mid \{v,w\} \in E\}\!\}$. Notice, then, that ${\sf eb}^t(G)$ is completely determined by the multiset 
		$\Gamma(G) = \{\!\{{\sf eb}^t(v) \mid v \in V\}\!\}$. 
		For each element $\tau \in \Gamma(G)$, we write  ${\cal C}_\tau$ for the multiset of colors computed by EB-1WL after $t$ steps that belong to $\tau$. 
		
		Define the following formula from CFOC$^3$:
		$$ 
		\phi_G := \big( \bigwedge_{\tau \in \Gamma(G)} \exists^{= \ell_\tau} x \, \phi_\tau(x)\big) \, \wedge \, \exists^{= \sum_{\tau \in \Gamma(G)} \ell_\tau} x (x=x), 
		$$
		where $\ell_\tau$ is the multiplicity of $\tau$ in $\Gamma(G)$ and $\phi_\tau(x)$ is defined as follows: 
		$$
		\phi_\tau(x) \ = \ \bigg(\bigwedge_{c \in {\cal C}_\tau} \exists^{= q_c} y 
		\, \big(E(x,y) \wedge \phi_c(x,y)\big)\bigg) \, \wedge \, 
		\exists^{= \sum_{c \in {\cal C}_\tau} q_\tau} y \, E(x,y), 
		$$
		where $q_\tau$ is the cardinality of $\tau$ in $\Gamma(G)$ and $\phi_c(x,y)$ is the CFOC$^3$ formula from Lemma \ref{lem:types1}. 
		Notice that $\phi_G$ has quantifier depth bounded by $t+2$. 
		
		It is easy to see that $G' \models \phi_G \Leftrightarrow \Gamma(G) = \Gamma(G')$. Since $G \models \phi_G$ and ${\sf tp}^{t+2}(G) = {\sf tp}^{t+2}(G')$, we conclude that $G' \models \phi_G$, and hence $\Gamma(G) = \Gamma(G')$. Since $\Gamma(G)$ determines ${\sf eb}^t(G)$, we conclude that ${\sf eb}^t(G) = {\sf eb}^t(G')$. 
		
	\end{proof}

	\subsection{Distinguishing power based on homomorphism counts}
	\label{sec:homomorphism-characterization}
	
	Next, we study the expressiveness of EB-1WL through the lens of homomorphism
	counts, which has recently become an important theme in the study of the WL
	test (cf.,~\citet{DBLP:conf/icalp/DellGR18,DBLP:conf/nips/BarceloGRR21,DBLP:conf/icml/JinBCL24,bao2025homomorphismcountsstructuralencodings}).
	
	Formally, a {\em homomorphism} from a graph $G = (V,E)$ to a graph $H = (V',E')$ is a mapping $h : V \to V'$ such that $\{h(u),h(v)\} \in E'$ for every edge $\{u,v\} \in E$.
	The connection between WL and homomorphism counts is as follows: two graphs are distinguishable by $k$WL if and only if they differ in the number of homomorphisms from some graph of {\em treewidth} at most $k$ \citep{DBLP:journals/jgt/Dvorak10,DBLP:conf/icalp/DellGR18}. For $k = 1$, this corresponds to the class of trees, and for $k = 2$ to the class of series-parallel graphs.

	We show that distinguishability by EB-1WL is at least as powerful as
	distinguishing graphs by homomorphism counts from the class of {\em chordal}
	graphs of treewidth two, which strictly lies between the classes of graphs of
	treewidth one and two.  This offers additional support for viewing EB-1WL as a
	natural and well-founded counterpart to the standard 1WL test.

	\begin{theorem} \label{theo:counts}
		If two graphs have a different number of homomorphisms from some chordal graph of treewidth at most 2, they are distinguishable by EB-1WL.
	\end{theorem}
	\begin{proof}
		 It is well-known that chordal graphs admit a \emph{perfect elimination order}: an ordering of its nodes such that every node $v$ and its neighbors that go before $v$ in the order form a clique~\citep{rose1970triangulated}. Graphs of tree-width 2 cannot have cliques larger than a triangle; this means that for any chordal graph $H$ of tree-width 2 there exists an ordering $v_1, \ldots, v_m$ of its nodes such that for any $k\in\{1, \ldots, m\}$,  one of the following holds:
		\begin{itemize}
			\item (a) $v_k$ is not connected to any node out of $v_1, \ldots,  v_{k-1}$;
			\item (b) $v_k$ is connected to exactly one node out of $v_1,\ldots, v_{k -1}$;
			\item (c) $v_k$ is connected to exactly two nodes $v_i, v_j\in\{v_1, \ldots, v_{k-1}\}$, and $v_i$ and $v_j$ are also connected.
		\end{itemize}
		
		Let $\eb$ be the stable EB-1WL coloring of $G$.  
		Let us show that the multiset $\eb(G)$ of $\eb$-labels of ordered edges uniquely determines the number of homomorphisms from $H$ to $G$, for any tree-width 2 chordal graph $H$. This means that if two graphs $G_1$ and $G_2$ have a different number of homomorphisms from some graph $H$ like that, they will be distinguished by the EB-1WL on the stage where colorings of both graphs stabilize. 
		
		For a node $u$, define:
		\[\eb(u) = \bleftm \eb(u,w) \mid w\in N(u)\brightm.\]
		Note that $\eb(u,v)$, as it is stable, uniquely determines induced labels of $u$ and $v$ through \eqref{eq_u} and \eqref{eq_v}. Moreover, $\eb(G)$ uniquely determines the multiset $\eb(V) = \bleftm \eb(u) \mid u\in V\brightm$. Namely, we go through all $\eb(u,v)$, compute $\eb(u)$ from it, which in turn determines the degree of $u$. We then divide the number of occurrences of $\eb(u)$ by the degree.
		
		Consider any labeling of edges of $H$ by labels from $\eb(G)$, and of its nodes by labels from $\eb(V)$. We show that for any such labeling, the number of homomorphisms from $H$ to $G$ that ``preserve'' this labeling is determined just by the multiset $\eb(G)$. The total number of homomorphisms is hence also determined by $\eb(G)$, since we can go through all possible labelings of $H$ and sum up homomorphisms for all of them.
		
		Here, ``preserve'' formally means that (a) a node $v_j$ with a label $\ell$ goes into a node in $G$ that has this $\eb$-label (b) if $v_i, v_j$ is an edge of $H$ (where $i < j$ are indices of these nodes in the perfect elimination order), and if 
		$v_i$ and $v_j$ go to some nodes $u_i, u_j$ in $G$, respectively, 
		then the label of the edge $v_i, v_j$ in $H$ has to be equal to $\eb(u_i, u_j)$.
		
		We show that we can compute the number of homomorphisms that preserve a given labeling of $H$, just knowing $\eb(G)$, by first computing how many ways we can define the image of $v_1$, then the image of $v_2$, then of $v_3$, and so on.

		As for $v_1$, it is assigned a label $\ell$ in the labeling; we know the multiset of $\eb$-labels of the nodes of $G$, which determines the number of ways we can define the image of $v_1$ (this is the number of nodes of $G$ that have label $\ell)$.

		Now, assume that we have defined images of $v_1, \ldots, v_{k - 1}$. Now, it's $v_k$'s turn. Firstly, it is possible that $v_k$ is not connected to any node among $v_1, \ldots, v_{k - 1}$. Then we can freely map $v_k$ to any node of $G$ that has the same label $\ell$ as assigned to $v_k$ in the coloring.
		We just have to multiply the current number of homomorphisms by the number of nodes in $G$ with this label $\ell$.

		If $v_k$ is connected to a single node $v_i$, $i< k$, and $v_i$ is already mapped to some node $u_i$, then we have to map $v_k$ to some node $u_k$ that is connected to $u_i$ and such that the $\eb(u_i, u_k)$ coincides with the color of the edge $v_i, v_k$ in $H$ (this color also determines the label of $u_k$ which has to be consistent with the label that $v_k$ has in the labeling of $H$, otherwise the number of homomorphisms is just $0$). The number of ways to choose such $u_k$ is thus determined by $\bleftm \eb(u_i, w) \mid w \in N({u_i})\brightm = \eb(u_i)$, which in turn equals the label of $v_i$ in the labeling of $H$. We multiply the current number of homomorphisms by the number of occurrences of the label of the edge $v_i, v_k$ in the label of $v_i$.

		The same argument holds for the third case when $v_k$ is connected to previous nodes $v_i, v_j$ that are connected by an edge. The number of ways to choose the image of $v_k$ is the number of occurrences of the pair $(\ell_{ik}, \ell_{jk})$ in  \eqref{eq_common} in the label of the edge $\ell_{ij}$, where $\ell_{ik}, \ell_{jk},$ and $\ell_{ij}$ are labels of edges $v_i, v_k$, $v_j, v_k$, and $v_i, v_j$ in $H$, respectively.
	\end{proof}
	
	\section{Edge-based Graph Neural Networks}
	\label{sec:ebgnn}
	
	Now we introduce the \emph{EB-GNN architecture}, a message-passing framework whose expressive power coincides with that of the EB-1WL test. 
	
	Formally, a $d$-dimensional  \emph{EB-GNN} $\mathcal{T}$ with $L > 0$ layers is specified by parameters 
	$a_i, b_i, c_i, u_i, v_i \in\mathbb{R}^d$ and $A_i, C_i, U_i, V_i \in\mathbb{R}^{d\times d}$,
	for $i = 1, \ldots, L$. Given a graph $G = (V,E)$, the EB-GNN~$\mathcal{T}$ assigns to each ordered edge $(u,v)$ with $\{u,v\} \in E$ and each layer $0 \leq i \leq L$ a feature vector $f^{(i)}(u,v) \in \mathbb{R}^{d}$.
	
	At the input layer, we set\footnote{When nodes or edges have additional input features, one can take them into account while defining $f^{(0)}(u,v)$, see Section \ref{sec:exp}.}  
	\begin{equation}
		\label{eq_init}
		f^{(0)}(u,v) = \begin{pmatrix}
			1 & 0 & \ldots & 0
		\end{pmatrix}^T \in \mathbb{R}^{d}.    
	\end{equation}
	
	For $1\leq i \le L$, the update rules are given by
	\begin{align}
		\label{eq_left}
		\Cline[blue]{\alpha^{(i)}(u)}
		&= \sum_{x \in N(u)} 
		\operatorname{ReLU}\!\Big( A_{i} \cdot f^{(i-1)}(u,x) + a_{i} \Big), \\[6pt]
		\label{eq_center}
		\Cline[yellow]{\beta^{(i)}(u,v)}
		&= \sum_{y \in N(u) \cap N(v)} 
		\operatorname{ReLU}\!\Big( B_{i} \cdot
		\begin{pmatrix}
			f^{(i-1)}(u,y) \\
			f^{(i-1)}(v,y)
		\end{pmatrix}
		+ b_{i}\Big), \\[6pt]
		\label{eq_right}
		\Cline[purple]{\gamma^{(i)}(v)}
		&= \sum_{z \in N(v)} 
		\operatorname{ReLU}\!\Big( C_{i} \cdot f^{(i-1)}(v,z)  + c_{i}\Big), \\[6pt]
		\label{eq_overall}
		g^{(i)}(u,v) 
		&= f^{(i-1)}(u,v) + \alpha^{(i)}(u) + \beta^{(i)}(u,v) + \gamma^{(i)}(v),\\
		\label{eq_ffn}
		f^{(i)}(u,v)  &= g^{(i)}(u,v)  + \ffn_{U_i,u_i, V_i, v_i}(g^{(i)}(u,v)),
	\end{align}
	where 
	$\operatorname{ReLU}(x) = \max\{0,x\}$ is applied coordinate-wise, and $\ffn_{U,u,V,v}(x) = V \cdot \operatorname{ReLU}(Ux + u) + v$. 
	
	The overall output of $\mathcal{T}$ on $G$ is defined as
	\begin{equation}
		\label{eq_pooling}
		\mathcal{T}(G) \;=\; \sum_{\{u,v\} \in E} f^{(L)}(u,v).    
	\end{equation}
	We say that graphs $G$ and $H$ are \emph{distinguishable by EB-GNNs} if there exists an EB-GNN $\mathcal{T}$ with $\mathcal{T}(G) \neq \mathcal{T}(H)$.  
	
	It is immediate that the distinguishing power of EB-GNNs cannot exceed that of EB-1WL. Notably, we can show that this upper bound is tight.
	\begin{theorem}
		\label{thm_morris}
		Pairs of graphs distinguishable by EB-1WL are also distinguishable by EB-GNNs.
	\end{theorem}
	\begin{proof}
	We show that EB-GNNs require dimension $O(m+t)$, where $m$ is the
		number of edges and $t$ the number of triangles. 
		The theorem is deduced from the following lemma.
		\begin{lemma}
			\label{lemma_ind}
			Let $f_1, \ldots, f_n\in\mathbb{R}^d$ be $n$ distinct vectors. Then there exists a matrix $A\in\mathbb{R}^{n\times d}$ and a vector $b\in\mathbb{R}^n$ such that the vectors
			\[g_1 = \operatorname{ReLU}(A f_1 + b), \ldots, g_n  =  \operatorname{ReLU}(A f_n + b)\]
			are linearly independent.
		\end{lemma}
		Indeed, consider any two graphs $G$ and $H$, distinguishable by the EB-1WL test. We construct an EB GNN $\mathcal{T}$ that distinguishes $G$ and $H$.

		Consider the disjoint union of $G$ and $H$. Let $m$ be the number of ordered edges of this union, and $t$ be the number of ordered triangles. The dimension of $\mathcal{T}$ will be
		\[d = 1 + 2 m + t.\]
		We show that there exists a choice of parameter matrices such that, for any $i$, a) EB-1WL labels after $i$ iterations are in a one-to-one correspondence with feature vectors  $f^{(i)}(u,v)$; b) all coordinates of $f^{(i)}$, except the first one, are 0s.
		
		For $i = 0$, this holds because all edges initially have the same EB-1WL label, and because of \eqref{eq_init}.

		Assume that it holds after $i - 1$ iterations. We use  Lemma~\ref{lemma_ind} in (\ref{eq_left}--\ref{eq_right}) to map distinct feature vectors $f^{(i - 1)}(u,v)$ (there are at most $m$ of them) or their pairs as in \eqref{eq_center}) (there are at most $t$ such pairs) to linearly independent vectors. This ensures that we obtain different sums for different multisets of terms in (\ref{eq_left}--\ref{eq_right}). We can use 3 blocks of $m, m$ and $t$ disjoint coordinates that are ```free'' in vectors $f^{(i-1)}(u,v)$  so that the sum in \eqref{eq_overall} uniquely determines the whole 4-tuple in the definition of the updated EB-1WL feature.
		
		We can then define a feed-forward network to injectively map vectors $g^{(i)}(u,v)$ to vectors $f^{(i)}(u,v)$ that have all coordinates, except the first one, equal to $0$.  There exists a vector $w = (w_1, \ldots, w_d)\in\mathbb{R}^d$ such that $\langle w, g^{(i)}(u,v)\rangle \neq \langle w, g^{(i)}(u',v')\rangle$ whenever $g^{(i)}(u,v)\neq g^{(i)}(u',v')$ (for each of the finitely many pairs of distinct $g^{(i)}$-vectors, the set of $w$ for which we have equality is a hyperplane in $\mathbb{R}^d$). Hence, it is enough to realize the following linear transformation by a FFN:
		\[g^{(i)}(u,v) \mapsto \begin{pmatrix}
			w_1 & w_2 & \ldots & w_d \\ 0 & 0 & \ldots & 0 \\
			\vdots &   \vdots & \ddots & \vdots\\
			0 & 0 & \ldots & 0 
		\end{pmatrix} g^{(i)}(u,v) = W g^{(i)}(u,v).\]
		This can be achieved by the following FFN:
		\[x \mapsto W\relu(x + u_i)  - Wu_i,\]
		where $u_i\in\mathbb{R}^d$ is an arbitrary vector such that $u_i \ge g^{(i)}(u,v)$ coordinate-wise for every ordered edge $(u,v)$. Indeed, then we obtain:
		\[ x\mapsto W\relu(x + u_i)  - Wu_i = W(x + u_i) - Wu_i =  Wx\]
		for any $x$ of the form $x = g^{(i)}(u,v)$, as required.
		
		Now, assume that we are  at an iteration when  $G$ and $H$ have different multisets of EB-1WL labels. We now need to do one more iteration that maps distinct feature vectors into linearly independent vectors so that $G$ and $H$ will have different final representations in \eqref{eq_pooling}. We set all matrices and bias vectors in (\ref{eq_left}--\ref{eq_right}) to 0 so that $g^{(i)}(u,v) = f^{(i - 1)}(u,v)$. We then use Lemma \ref{lemma_ind} again to construct a matrix $U$ and a vector $u$ such that the following transformation:
		\begin{equation}
			\label{eq_image}
			f^{(i - 1)}(u,v) \mapsto \relu(U f^{(i - 1)}(u,v) + v)
		\end{equation}
		maps distinct vectors into linearly independent ones.
		There are at most $m$ distinct feature vectors, which means we can use coordinates that are 0 in $g^{(i)}(u,v)$ (by the invariant b)) for the image of \eqref{eq_image} so it does not interfere with the first coordinate of $g^{(i)}(u,v)$ in \eqref{eq_ffn} . By setting $V = Id, v = 0$, we obtain a FFN that realizes this transformation.
		\begin{proof}[Proof of Lemma \ref{lemma_ind}]
			Since $f_1, \ldots, f_n$ are distinct, there exists $w\in\mathbb{R}^d$ such that $\langle f_1, w\rangle, \ldots, \langle f_n, w\rangle$ are distinct (because the set of $w$ such that $\langle f_i, w\rangle = \langle f_j, w\rangle$ for some $i\neq j$ is a union of finitely many hyperplanes, not covering the whole $\mathbb{R}^d$). Without loss of generality, 
			\[\langle f_1, w\rangle <\langle f_2, w\rangle <\ldots <\langle f_n, w\rangle.\]
			For $i = 1, \ldots, n$, let $\gamma_i$ be some number between $\langle f_{i-1}, w\rangle$ and $\langle f_{i}, w\rangle$ (for $i = 1$, this is some number smaller than $\langle f_1, w\rangle$).
			Define
			\[ A = \begin{pmatrix}
				w \\ \vdots \\ w
			\end{pmatrix}, \qquad b = \begin{pmatrix}
				- \gamma_1 \\ -\gamma_2 \\ \vdots \\ -\gamma_n
			\end{pmatrix}.\]
			
			Observe that in the vector
			\[A f_i + b= \begin{pmatrix}
				\langle f_i, w\rangle - \gamma_1 \\
				\langle f_i, w\rangle - \gamma_2 \\
				\vdots \\
				\langle f_i, w\rangle - \gamma_n
			\end{pmatrix}\]
			the first $i$ coordinates are strictly positive, and the other coordinates are strictly negative. Hence, the vector $g_i =\operatorname{ReLU}(A f_i + b)$ is a vector where the first $i$ coordinates are strictly positive, and the rest are 0s. Therefore, $g_1, \ldots, g_n$ are linearly independent. 
		\end{proof}
		
		This finishes the proof of the theorem. 
	\end{proof}
	\paragraph{Algorithmic implementation.} 
	In our implementation we perform a preprocessing step in which we enumerate all triangles in time $O(\alpha m)$ using the algorithm of ~\citet{chiba1985arboricity}. 
	Since each triangle $(u,v,y)$ corresponds to an edge $(u,v)$ and a node $y\in N(u)\cap N(v)$,
	we can obtain the sets $N(u)\cap N(v)$ required in Eq.~\eqref{eq_center} for all
	edges $(u,v)$ in time $O(\alpha m)$.  After that, each iteration of EB-GNN takes
	time $O(m + t)$, where $t$ is the number of triangles---implying that in
	practice we might obtain better running times per iteration than in Proposition~\eqref{prop_time} because $t = O(\alpha m)$. However, due to this preprocessing step we require memory $O(m + t)$ (instead of just $O(m)$).

	\section{Experimental Evaluation}
	\label{sec:exp}
	\begin{table}[t]
		\caption{Empirical results on expressivity datasets. GIN + C$_3$ is an MPNN that uses triangle subgraph counts as additional node features. MPNN and 2WL results on \texttt{BREC} are from \citet{BREC}. For details on runtime constants see Table~\ref{tab:qm9}.}
		\label{tab:expressivity}
		
		\resizebox{\columnwidth}{!}{
			\begin{tabular}{lcccccc}
				\toprule
				& Runtime & \tablecell{\texttt{CSL}\\{\tiny Accuracy ($\uparrow$)}} & \multicolumn{4}{c}{\tablecell{\texttt{BREC}\\ {\tiny\# Distinguishable Graph Pairs ($\uparrow$)}}} \\
				\cmidrule(lr){3-6}
				Model & & & \texttt{Basic} & \texttt{Reg.} & \texttt{Ext.} & \texttt{CFI} \\
				\midrule
				MPNN       & $\mathcal{O}\left(m\right)$ & $10\%$ & 0 & 0 & 0 & 0 \\
				MPNN + C$_3$ & $\mathcal{O}\left(\alpha m \right)$ & $20\%$ & 0 & 0 & 0 & 0 \\ \hline
				$2$WL       & $\mathcal{O}\left(n^3\right)$ & -- & 60 & 50 & 100 & 60 \\
				$I^2$-GNN & $\mathcal{O}\left(nsd^2\right)$ & -- & 60 & 100 & 100 & 21 \\
				DRFWL & $\mathcal{O}\left(nd^4\right)$ & -- & 60 & 50 & 99 & 0 \\
				4-$\ell$-GIN & $\mathcal{O}\left(nd^3\right)$ & $60\%$ & 60 & 100 & 95 & 2 \\ \hline
				NC-GNN      & $\mathcal{O}\left(\alpha m \right)$ & $20\%$ & 52 & 48 & 0 & 0 \\
				EB-GNN (ours)     & $\mathcal{O}\left(\alpha m \right)$ & $20\%$ & 59 & 48 & 60 & 0 \\
				\bottomrule
		\end{tabular}}
	\end{table}

	Now we empirically evaluate EB-GNN.  The primary goal of our experiments is to show that EB-GNN provides a fast and expressive general-purpose GNN architecture. Therefore, we evaluate EB-GNN across tasks from diverse domains. We compare EB-GNN against Message Passing Neural Networks (MPNNs), another widely used general-purpose architecture, as well as against state-of-the-art models specifically optimized for each corresponding task. Our implementation is available at \url{https://anonymous.4open.science/r/EdgeBasedGNNs-B526}.

	We focus on graph-level predictions and predictions on existing edges. While graph-level prediction tasks are a well-established use case for predictive GNNs \citep{MorrisAAAI19, WL_Loopy, HyMN}, prediction tasks on existing edges have received less attention. This latter task is relevant in chemistry: rather than relying on costly molecular simulations, GNNs can directly predict quantum mechanical properties, e.g., the bond length between atoms \citep{QMD}.
	
	\subsection*{Experiment setup.} Next, we describe the datasets and baseline models used for comparison. We evaluate EB-GNN on two synthetic datasets designed to measure its practically realized expressivity. Additionally, we assess performance on three real-world datasets: two composed of small molecular graphs and one consisting of large cybersecurity graphs. Consistent with previous observations that real-world graphs have low arboricity, we find that these datasets exhibit small arboricity: $\leq 3$ for molecular graphs and $\leq 15$ for large cybersecurity graphs with each dataset having $<4$ mean arboricity. For each dataset, we compare EB-GNN against both general-purpose models and state-of-the-art task-specific baselines. Details on model and experiment setup are in App.~\ref{sec:appendix_experiments}
	
	\paragraph{Synthetic datasets for measuring expressivity.} We evaluate the empirical expressivity of EB-GNN on two synthetic datasets: \texttt{CSL} \citep{relational_pooling, Benchmarking-GNNs} and \texttt{BREC} \citep{BREC}. The \texttt{CSL} dataset consists of 150 graphs with 41 nodes each, grouped into 10 distinct isomorphism classes. These classes, defined by skip connections between Hamiltonian cycles, are all indistinguishable by $1$WL. The graph-level task is to classify each graph according to its isomorphism class. The \texttt{BREC} dataset comprises pairs of graphs that are indistinguishable by $1$WL but distinguishable by $3$WL. For each pair, the graph-level task is to compute embeddings that correctly differentiate the two graphs.
	
	\begin{table}[t]
		\centering
		\caption{Results on \texttt{MalNet-Tiny}. Top three models as  \first{$1^{\text{st}}$}, \second{$2^{\text{nd}}$}, \third{$3^{\text{rd}}$}. Baselines from \citet{HyMN}.}
		\label{tab:malnet}
		\begin{tabular}{lc}
			\toprule
			Method & MalNet-Tiny {\tiny Accuracy ($\uparrow$) } \\
			\midrule
			MPNN   & 91.10 {\tiny $\pm 0.98$} \\
			HyMN & \second{92.84} {\tiny $\pm 0.52$} \\ 
			GPS (Perf.)   & 92.14 {\tiny $\pm 0.24$} \\ 
			GPS (BigBird)  & 91.02 {\tiny $\pm 0.48$} \\ 
			GPS (Transf.) & 90.85 {\tiny $\pm 0.68$} \\     
			NC-GNN &  \third{$92.50$} {\tiny $\pm 0.56$} \\
			\midrule 
			EB-GNN & \first{$93.30$} {\tiny $\pm 0.66$} \\
			\bottomrule
			\vspace{-0.5cm}
		\end{tabular}
	\end{table}
	
	\begin{table*}[t]
		\centering
		\caption{MAE ($\downarrow$) on \texttt{QMD}. Best model per task marked \first{blue}. D-MPNN results from \citet{QMD}.}
		\label{tab:qmd}
		\setlength\tabcolsep{3pt}
		\resizebox{\textwidth}{!}{
			\begin{NiceTabular}{lllllllll}[create-large-nodes]
				\toprule
				& \multicolumn{4}{c}{\textbf{Edge-level Tasks} (MAE $\downarrow$)} & \multicolumn{4}{c}{\textbf{Graph-level Tasks} (MAE $\downarrow$)} \\
				\cmidrule(lr){2-5} \cmidrule(lr){6-9}
				Model & \tablecell{Bond\\Index\\ (unitless)} & \tablecell{Bond\\Length\\ (\AA)} & \tablecell{Bonding\\Electrons\\(e)} & \tablecell{Natural\\Ionicity\\(unitless)} & \tablecell{IP} & \tablecell{EA} & \tablecell{Dipole\\Moment\\(debye)} & 
				\tablecell{Traceless\\Quad. Mom.\\(debye \AA)} 
				\\
				& $\times 10^{-3}$ & $\times 10^{-3}$ & $\times 10^{-2}$ & $\times 10^{-4}$ & $\times 10^{-3}$ & $\times 10^{-3}$ & $\times 10^{-1}$ & $\times 10^{0}$ \\
				\midrule
				D-MPNN & $6.65$  & $4.48$  & $1.46$  & $9.00$ & \first{$4.29$} & \first{$4.06$} & $4.59$ & $1.62$ 
				\\
				NC-GNN & 4.82 {\tiny $\pm 0.04$} & 3.33 {\tiny $\pm 0.03$} & 1.28 {\tiny $\pm 0.01$} & 5.39 {\tiny $\pm 0.25$} & $5.3$ {\tiny $\pm 0.03$} & $4.35$ {\tiny $\pm 0.03$} & $4.45$ {\tiny $\pm 0.02$} &$ 1.57$ {\tiny $\pm 0.01$} \\
				EB-GNN
				& \first{$4.42$} {\tiny $\pm  0.01$} & \first{$3.17$} {\tiny $\pm  0.011$} &  \first{$1.19$} {\tiny $\pm  0.018$} & \first{$4.64$} {\tiny $\pm  0.01$} & $5.49$ {\tiny $\pm 0.26$} & $4.46$ {\tiny $\pm  0.12$} & \first{$4.33$} {\tiny $\pm 0.05$} & \first{$1.55$} {\tiny $\pm 0.01$} \\
				\bottomrule
				\CodeAfter
				\tikz \draw[dashed, line width=0.6pt] (1-|6) -- (last-|6);
		\end{NiceTabular}}
	\end{table*}
	
	On the synthetic datasets, we compare EB-GNN against $2$WL to assess how much of $2$WL’s expressivity EB-GNN can replicate while maintaining asymptotically faster runtimes. In addition, we include a comparison with the MPNN GIN \citep{xu18} augmented with triangle subgraph counts as node features \citep{gsn}, referred to as MPNN + C$_3$.
	This comparison helps isolate how much of EB-GNN’s expressivity gain stems from its ability to count triangles, as captured by the $\beta$ aggregation  in Eq.~\eqref{eq_center}. We also compare against NC-GNN to determine whether our theoretical increase in expressivity (Thm.~\ref{theo:versus}) can also be measured empirically. Furthermore, we compare against all models used as baselines in other experiments for which we could find publicly available results.
	
	\medskip
	
	\paragraph{Molecular edge-level and graph-level tasks.} 
	We further evaluate EB-GNN on a range of molecular edge-level and graph-level prediction tasks.
	\citet{QMD} introduce the \texttt{QMD} dataset, which contains 65\,000 molecular graphs annotated with various quantum mechanical properties.  
	We assess EB-GNN on all four edge-level regression tasks from \texttt{QMD}.  
	We evaluate four diverse graph-level regression tasks from \texttt{QMD}, selected to represent varied objectives (most other graph-level tasks in \texttt{QMD} focus on predicting HOMO/LUMO gaps).
	We also evaluate EB-GNN on 12 graph-level regression tasks from the widely used \texttt{QM9} dataset \citep{MoleculeNet}, following common practice \citep{MorrisAAAI19, WL_Loopy}.  
	\texttt{QM9} comprises 130\,000 molecular graphs representing molecules of up to 9 atoms.  
	On average, graphs in both \texttt{QMD} and \texttt{QM9} contain fewer than 20 nodes.

	For \texttt{QMD} tasks, we compare EB-GNN against D-MPNN \citep{doi:10.1021/acs.jcim.9b00237, 10.5555/3045390.3045675}, a directed MPNN that has demonstrated strong performance on chemical prediction tasks \citep{doi:10.1021/jacs.2c01768, doi:10.1021/acs.jcim.1c00975} and is integrated into the widely adopted ChemProp framework \citep{ChemProp}.  
	Similar to EB-GNN, D-MPNN performs message passing on edges rather than nodes. Furthermore, we also compare against NC-GNN \citep{DBLP:journals/tmlr/0015YJ24} which we train using the same hyperparameter tuning procedure as EB-GNN.
	For \texttt{QM9}, we compare against a standard MPNN and several expressive GNN architectures: 1-2-3 GNN \citep{MorrisAAAI19}, DTNN \citep{DTNN, MoleculeNet}, NestedGNN \citep{zhang2021nestedgraphneuralnetworks}, I2-GNN \citep{huang2023boostingcyclecountingpower}, DRFWL \citep{zhou2023distancerestricted}, and the recent state-of-the-art 5-$\ell$GIN baseline \citep{WL_Loopy}.

	\paragraph{Cybersecurity.} To evaluate our model on a different domain with larger graphs, we conduct experiments on \texttt{MalNet-Tiny} \citep{MalNet}. 
	\texttt{MalNet-Tiny} consists of 5\,000 graphs with an average of over 1\,500 nodes, sampled from the \texttt{MalNet} dataset. 
	The graph-level task is to classify whether a function call is benign or belongs to one of four malicious classes (AdWare, Trojan, Addisplay, Downloader).  
	For comparison, we include a standard MPNN and the recently proposed subgraph GNN HyMN \citep{HyMN}. 
	We further compare against the graph transformer GPS \citep{GPS} using different attention mechanisms: Performer \citep{choromanski2022rethinkingattentionperformers}, Big Bird \citep{zaheer2021bigbirdtransformerslonger}, and the standard Transformer \citep{vaswani2023attentionneed}. Finally, we also compare against NC-GNN \citep{DBLP:journals/tmlr/0015YJ24} which we train in the same fashion as EB-GNN.
	
	\subsection*{Results.}  
	Recall that comparisons are made against both general-purpose MPNNs and state-of-the-art models for each dataset. 
	Our experiments demonstrate that EB-GNN consistently outperforms standard MPNNs and remains competitive with state-of-the-art approaches.

	\paragraph{Results on synthetic data.} Table~\ref{tab:expressivity} presents the results of our expressivity experiments on the synthetic datasets. 
	On \texttt{CSL}, EB-GNN achieves an accuracy of 20\%, outperforming a vanilla MPNN (10\% accuracy) and matching an MPNN augmented with triangle counts (MPNN + C$_3$, 20\% accuracy).  
	On \texttt{BREC}, the gains of EB-GNN cannot be attributed to triangle counting, as evidenced by the MPNN + C$_3$ results.  
	Remarkably, EB-GNN achieves performance on \texttt{Basic} and \texttt{Regular} graphs that is nearly identical to $2$WL, though it fails to distinguish any \texttt{CFI} pairs.  
	These results indicate that EB-GNN’s empirically realized expressivity extends well beyond triangle counting and, in many cases, approaches $2$WL expressivity while maintaining significantly faster runtimes.We outperform NC-GNN in two categories and tie it in the remaining three. Notably, on \texttt{Extension} graphs EB-GNN solves 60 instances while NC-GNN solves 0, highlighting our empirical increase in expressivity.
	
	\begin{table*}[t]
		\centering
		\caption{Normalized test MAE ($\downarrow$) on QM9 dataset. Top three models as  \first{$1^{\text{st}}$}, \second{$2^{\text{nd}}$}, \third{$3^{\text{rd}}$}. Table based on \citet{WL_Loopy}. For runtime, $n$ is the number of nodes, $m$ the number of edges, $c$ and $s$ are maximum size of subgraph sizes, $d$ the maximum degree, and $\alpha$ the arboricity.}
		\label{tab:qm9}
		{\footnotesize
			\setlength\tabcolsep{3pt}
			\resizebox{\textwidth}{!}{
				\begin{tabular}{lllllllll}
					\toprule
					Target (MAE $\downarrow$) & \multicolumn{7}{c}{Model}\\
					\cmidrule{2-9}
					& MPNN & 1-2-3 GNN & DTNN & NestedGNN & I2-GNN & DRFWL & $5$-$\ell$GIN{} & EB-GNN \\
					Runtime & $\mathcal{O}\left(m\right)$ & $\mathcal{O}\left(n^3\right)$ & $\mathcal{O}(m)$ & $\mathcal{O}\left(ncd\right)$ & $\mathcal{O}\left(nsd^2\right)$ & $\mathcal{O}\left(nd^4\right)$  & $\mathcal{O}\left(nd^5\right)$ & $\mathcal{O}\left(\alpha m \right)$ \\
					\midrule
					$\mu$ {\tiny$\left(\times 10^{-1}\right)$} & 4.93 & 4.76 & \first{2.44} & 4.28 & 4.28 & \third{3.46} & 3.50 {\tiny$\pm 0.11$}  & \second{3.15} {\tiny $\pm 0.02$}
					\\
					$\alpha$ {\tiny$\left(\times 10^{-1}\right)$} & 7.8 & 2.7 & 9.5 & 2.90 & 2.30 & \third{2.22} & \second{2.17} {\tiny$\pm 0.25$} & \first{2.08} {\tiny $\pm0.03$}
					\\
					$\varepsilon_{\mathrm{homo}}$ {\tiny$\left(\times 10^{-3}\right)$} & 3.21 & 3.37 & 3.88 & 2.65 & 2.61 & \third{2.26} & \first{2.05} {\tiny$\pm 0.05$}& \second{2.18} {\tiny $\pm 0.01$}
					\\
					$\varepsilon_{\mathrm{lumo}}$  {\tiny$\left(\times 10^{-3}\right)$} & 3.55 & 3.51 & 5.12 & 2.97 & 2.67 & \third{2.25} & \first{2.16}  {\tiny$\pm 0.04$} & \second{2.17} {\tiny $\pm0.02$}
					\\
					$\Delta(\varepsilon)$  {\tiny$\left(\times 10^{-3}\right)$} & 4.9 & 4.8 & 11.2 & 3.8 & 3.8 & \third{3.24} & \second{3.21} {\tiny$\pm 0.14$} &  \first{3.02}  {\tiny $\pm 0.01$}
					\\
					$R^2$ & 34.1 & 22.9 & 17.0 & 20.5 & 18.64 & \third{15.04} & \first{13.21} {\tiny$\pm 0.19$} & \second{13.83} {\tiny $\pm0.12$}
					\\
					$\mathrm{ZVPE}$ {\tiny$\left(\times 10^{-4}\right)$} & 12.4 & 1.9 & 17.2 & 2. & 1.4 & \third{1.7} & \second{1.27} {\tiny$\pm 0.03$} & \first{1.26} {\tiny $\pm0.01$}
					\\
					$U_0$ & 2.32 & 0.0427 & 2.43 & 0.295 & 0.211 & \third{0.156} & \first{0.0418} {\tiny$\pm 0.0520$} & \second{0.063} {\tiny $\pm0.002$}
					\\
					$U$ & 2.08 & \third{0.111} & 2.43 & 0.361 & 0.206 & 0.153 & \first{0.023} {\tiny$\pm 0.023$} &  \second{0.078} {\tiny$\pm 0.008$}
					\\
					$H$ & 2.23 & \third{0.0419} & 2.43 & 0.305 & 0.269 & 0.145 & \first{0.0352}  {\tiny$\pm 0.0304$} & \second{0.0564} {\tiny$\pm 0.0075$}
					\\
					$G$ & 1.94 & \second{0.0469} & 2.43 & 0.489 & 0.261 & 0.156 & \first{0.0118} {\tiny$\pm 0.0015$} & \third{0.076} {\tiny$\pm 0.006$}
					\\
					$C_v$ & 0.27 & 0.0944 & 2.43 & 0.174 & \second{0.0730} & \third{0.0901} & \first{0.0702} {\tiny$\pm 0.0024$} & 0.092 {\tiny$\pm 0.001$}
					\\
					\bottomrule
		\end{tabular}}}
	\end{table*}
	
	\paragraph{Molecular edge-level and graph-level tasks.} 
	Table~\ref{tab:qmd} presents the edge-level and graph-level results on \texttt{QMD}, compared against the baseline D-MPNN \citep{QMD}.  
	EB-GNN outperforms D-MPNN across all edge-level tasks, reducing the mean absolute error by 20\% to 50\% depending on the task. On graph-level tasks, EB-GNN and D-MPNN perform comparably, with EB-GNN slightly outperforming D-MPNN on two tasks and slightly underperforming on the other two. EB-GNN outperforms NC-GNN in six out of eight datasets and only loses to NC-GNN for datasets where D-MPNN is the best performing model. 
	Table~\ref{tab:qm9} summarizes the results on \texttt{QM9}, where we compare against several expressive GNNs. 
	The current state-of-the-art on this dataset is 5-$\ell$GNN \citep{WL_Loopy}, which significantly improved over previous architectures. 
	EB-GNN ranks as the best or second-best model on 11 out of 12 tasks, achieving comparable or better performance than 5-$\ell$GNN on 7 tasks (with at most 6\% lower accuracy) and outperforming it on 4 tasks.  
	Moreover, EB-GNN is substantially more computationally efficient than 5-$\ell$GNN: while 5-$\ell$GNN has asymptotic runtime $\mathcal{O}\left(n d^5\right)$, EB-GNN only requires $\mathcal{O}\left(\alpha m\right)$ time.
	In our experiments, EB-GNN is approximately 4 times faster (see App.~\ref{sec:appendix_experiments}).
	
	\paragraph{Results on graph-level tasks for malware detection.} 
	Table~\ref{tab:malnet} shows the results on \texttt{MalNet-Tiny}. EB-GNN outperforms all other models, including a standard MPNN, various graph transformers (GPS), and the subgraph GNN HyMN, 
	demonstrating that EB-GNN generalizes well to other domains and scales effectively to larger graphs.
	
	\section{Conclusion, Limitations and Future work}
	\label{sec:conclusion}
	
	We propose an edge-based message passing algorithm that combines high expressivity with near-linear runtime on sparse graphs. Our analysis fully characterizes its expressivity in terms of first-order logic and establishes a lower bound via homomorphism counting. Empirically, our architecture outperforms standard MPNNs while remaining competitive with more expressive models, at a substantially lower runtime. We observe that EB-GNN is a general-purpose architecture: while it achieves strong empirical results, it does not rely on specialized techniques used on top of basic architectures known to improve GNN performance, such as using subgraph counts \citep{gsn}, homomorphism counts \citep{DBLP:conf/nips/BarceloGRR21, pmlr-v202-welke23a,DBLP:conf/icml/JinBCL24} or positional encodings \citep{you2019positionawaregraphneuralnetworks,ying2021transformersreallyperformbad, pmlr-v202-ma23c,bao2025homomorphismcountsstructuralencodings}.
	
	\emph{Limitations}. Although EB-1WL and EB-GNN achieve near-linear running times for sparse graphs, their efficiency depends on arboricity and triangle enumeration, which can be expensive on very dense graphs, potentially limiting scalability in such settings. Furthermore, EB-GNN produces embeddings only for edges preventing direct application to node-level tasks and standard link-prediction methods that rely on node embeddings. We leave extending EB-GNN to these tasks as future work.

	\emph{Future work}. An interesting open problem left by our work is whether the converse of Thm.~\ref{theo:counts} holds, i.e., whether graphs distinguishable by EB-1WL are exactly those distinguishable by homomorphism counts from some chordal graph of treewidth 2. 
	Another line of future work concerns the tradeoff between expressiveness and generalization: recent results show that greater expressiveness need not harm generalization if matched to task demands and training data \citep{DBLP:journals/corr/abs-2505-11298}, and can even help when graphs are well separated by large margins \citep{DBLP:conf/iclr/LiG0025}. EB-GNNs strike a principled balance --- more expressive than NC-1WL, yet far cheaper than 2WL --- while showing strong results across benchmarks. Future work should broaden empirical evaluation to fully assess this balance of expressiveness, scalability, and generalization.
	
	\paragraph{Acknowledgments} Barcel\'o, Kozachinskiy, and Rojas are funded
	by the National Center for Artificial Intelligence CENIA
	FB210017, Basal ANID. Barcel\'o is also funded by ANID
	Millennium Science Initiative Program Code ICN17002. Kozachinskiy is also  supported by ANID Fondecyt Iniciaci\'on
	grant 11250060.
	Matthias Lanzinger acknowledges support by the Vienna Science and Technology Fund (WWTF) [10.47379/ICT2201].
	Neumann is funded by the Vienna Science and Technology Fund (WWTF) [Grant ID: 10.47379/VRG23013].

\appendix

\section{Relationship with the Chiba--Nishizeki algorithm}
\label{sec:relationship}

	We provide a thorough explanation of the algorithm by \citet{chiba1985arboricity}, state some of its properties which are important for our architecture and describe how the algorithm inspired our GNN architecture.
	
	\paragraph{Description of the Chiba--Nishizeki algorithm.}
	The algorithm by \citet{chiba1985arboricity} obtains as input an undirected, unweighted graph $G=(V,E)$ and it returns a list of all triangles in $G$.
	Here, we provide a slightly modified version of the algorithm, which is slightly easier to describe but has the same properties.
	
	Concretely, the algorithm works as follows. It iterates over all edges $(u,v)\in E$ and for each edge $(u,v)$ does the following: Suppose w.l.o.g.\ that $u$ is the lower-degree endpoint of the endpoint of the edge $(u,v)$, i.e., assume that $|N(u)| \leq |N(v)|$. Now the algorithm iterates over all neighbors $w\in N(u)$ and checks if $w\in N(v)$; if this is the case then it returns that $(u,v,w)$ is a triangle.
	
	Note that the algorithm returns all triangles: Clearly, for any triangle $(u,v,w)$ its edge $(u,v)$ will be considered in the outer loop and we will have that $w\in N(u)$ (thus $w$ is considered in the inner loop) and $w\in N(v)$ (satisfying the if-condition). Thus, the triangle $(u,v,w)$ will be reported.
	
	Further, the algorithm's running time is $O(\alpha m)$: Note that for each edge $(u,v)$ we spend time \[O(\min\{|N(u)|, |N(v)|\})\] since we only spend time proportional to the neighborhood size of the lower-degree endpoint of $(u,v)$. Here, we use that the check whether $w\in N(v)$ can be done in time $O(1)$ using hash maps.\footnote{This is where our version of the algorithm differs from the original. The paper by \citet{chiba1985arboricity} does not use hash maps and instead uses a slightly more complicated marking procedure.} 
	Thus, its total running time is given by $\sum_{(u,v)\in E} O(\min\{|N(u)|, |N(v)|\})$ 
	and \citet[Lemma~2]{chiba1985arboricity} showed that this quantity is bounded by 
	$O(\alpha m)$.
	
	\paragraph{Properties of the Chiba--Nishizeki algorithm.}
	The algorithm has several important properties:
	\begin{enumerate}
		\item The algorithm lists all triangles in time $O(\alpha m)$. This implies that graphs with arboricity~$\alpha$ contain at most $O(\alpha m)$~triangles. This allows us to bound the number of messages we need to send in our architecture due to triangles by $O(\alpha m)$.
		\item The running time of the Chiba--Nishizeki algorithm is optimal under standard assumptions from the complexity theory community~\citep{kopelowitz2016higher,williams2020monochromatic}. Thus, the preprocessing time of our algorithm cannot be improved if all triangles need to be enumerated.
		\item It is well-known in the algorithms community that the class of graphs with arboricity $O(1)$ contains natural graph families, such as planar graphs, minor-closed families, and preferential attachment graphs, among others. Thus, for all graphs from these families, all triangles can be enumerated in time~$O(n)$ and thus in time linear in the size of the input graph.
		See, e.g., \citet[Lemma~1]{chiba1985arboricity} for the planar case.
		\item It holds that $\alpha = O(\sqrt{m})$~\citet[Lemma~1]{chiba1985arboricity}. As a consequence, the arboricity can be at most a $O(\sqrt{n})$ factor larger than the average degree in the graph. Indeed, this is the case for the simple example consisting of a path with $n-\sqrt{n}$~vertices, which is connected to a clique with $\sqrt{n}$~vertices; this graph has $m=O(n)$ edges and thus average degree $O(1)$ but its arboricity is $O(\sqrt{n})$ due to the clique.
		\item The arboricity is highly related to other graph parameters, such as the \emph{degeneracy} or the \emph{densest subgraph}~\citep{nash1961edge} and differs from them by at most a factor of~2.
		\item In practice, it is well-known that practical networks have very small arboricities~\citep{eppstein2010listing}. Indeed, in all 39 datasets considered by \cite{eppstein2010listing}, the arboricity is at most 201 even though their largest graph has 3.7 million nodes and 16.5 million edges. On 32/39 of their datasets, the arboricity is less than 60.
		We note that \cite{eppstein2010listing} report the degeneracy, which is an upper bound on the arboricity.
	\end{enumerate}
	
	\paragraph{Relationship to our architecture.}
	Next, we briefly describe how the Chiba--Nishizeki algorithm inspired our architecture. Recall that the algorithm iterates over all $(u,v)$; then it iterates over all $w\in N(u)$ and checks whether $w\in N(v)$ to see if it should report a triangle. In other words, it only reports triangles for vertices~$w$ such that $w\in N(u) \cap N(v)$.
	
	When looking at this procedure from the perspective of distributed computing, one can view this as follows: Each edge~$(u,v)$ aggregates the neighborhoods $N(u)$ and $N(v)$ and then computes their intersection $N(u) \cap N(v)$ to obtain in which triangles it appears in.
	
	Indeed, this distributed point of view is the motivation for our Equations~\eqref{eq_left}, \eqref{eq_center} and \eqref{eq_right}: 
	Equation~\eqref{eq_left} aggregates the embeddings of neighbors $N(u)$ of $u$,
	Equation~\eqref{eq_right} aggregates the embeddings of neighbors $N(v)$ of $v$, and 
	Equation~\eqref{eq_center} aggregates the embeddings of all edges $(u,w)$ and $(v,w)$ such that $w\in N(u)\cap N(v)$ forms a triangle with $u$ and $v$.

	\section{More Details on Experiments}
	\label{sec:appendix_experiments}
	
	We provide additional information on our experimental procedure and more detailed results. 
	
	\paragraph{Model.} We have defined graph as purely existing of nodes and edges $G = (V, E)$ . However, real-world datasets often use node features and edge features to encode additional information in the graph for all. For every directed  edge $(u , v) \in E$, we incorporate the node features of $X_u, X_v$ of $u,v$ and the edge features $W_{(u,v)}$ of $(u,v)$ into our model by initializing the edge embedding ${\cal T}^{(0)}(G,(u,v))$ with them. We use three different embedding encoders $\text{ENC}$ to map both the node features and the edge features to the embedding dimension of our GNN
	
	$${\cal T}^{(0)}(G,(u,v)) = \text{ENC}_{\text{left}}\left(X_u\right) + \text{ENC}_{\text{right}}\left(X_v\right) + \text{ENC}_{\text{edge}}\left(W_{(u,v)}\right).$$
	
	After obtaining the initial edge embedding, we perform multiple iterations of edge based message passing. After final iteration $t$, we pool edge embeddings into the shape required by the task. For graph-level tasks, we experiment with three methods. We compute graph-level embeddings by either summing 
	\[{\cal T}_{\text{SUM}}(G) = \sum\limits_{(u,v) : \{u,v\}\in E} {\cal T}^t(G,(u,v)),\]
	computing the mean, or computing the mean scaled by the number of nodes (mimicking sum pooling in MPNNs) 
	\[
	\begin{aligned}
		{\cal T}_{\text{MEAN}}(G)     &= \tfrac{1}{|E|} {\cal T}_{\text{SUM}}(G), 
		&\qquad 
		{\cal T}_{\text{NODESUM}}(G) &= \tfrac{|V|}{|E|} {\cal T}_{\text{SUM}}(G).
	\end{aligned}
	\]
	
	For edge-level tasks, note that we compute embeddings of directed edges whereas tasks we worked on were on \emph{undirected} regression edges. We experiment with two methods of performing undirected edge predictions. First, we simply sum the embedding of the two directed edges and perform a prediction on this combined embedding. Second, we make a prediction for both directions and combine these predictions by computing the mean.
	
	All our models were implemented in PyTorch Geometric \citep{PyTorch, PyTorch-Geometric}. All our models were trained on servers with one NVIDIA GeForce RTX 3080 GPU (10 GB VRAM) and 64 GB of RAM. When training a model we evaluate its performance after every epoch on the validation and test set. This performance is measured in the metric that is most commonly used on that dataset. After training, we report the validation and test performance in the epoch with the best validation performance.
	For real-life datasets, we perform hyperparameter tuning where we pick the hyperparameter combination based on the best validation performance. We train a model with this hyperparameter combination multiple times on different seeds reporting the mean and standard deviation of the test metric. 
	
	In general, all our models are trained with a Cosine learning rate scheduler and a learning rate of $0.001$ (except on \texttt{BREC} where we use the procedure provided by \citet{BREC}). The final prediction is made by a two-layer MLP. On all real-life datasets EB-GNN uses both skip connections and feed-forward layers, but not on the synthetic datasets \texttt{CSL} and \texttt{BREC}. Below, we discuss the different setup we used for each dataset. For more details, please consider our code at \url{https://anonymous.4open.science/r/EdgeBasedGNNs-B526/}. 
	
	\paragraph{\texttt{CSL}.} We train all models for 1000 epochs with a Cosine learning rate scheduler. All GNNs have 5 layers and an embedding dimension of 64.
	
	\paragraph{\texttt{BREC}.} We train and evaluate our model with the procedure provided by \citet{BREC}. Our EB-GNN model has 10 layers, an embedding dimension of 16 and uses some pooling.  We have observed that using the output of  Eq.~\ref{eq_overall} instead of Eq.~\ref{eq_ffn} as edge embeddings leads to better results on the extension graphs. We believe that this is due to numerical issues caused by the large number of layers and small embedding dimension (which is necessary to fit the model into GPU memory). Our MPNN models also have 10 layers but use an embedding dimension of 64.
	
	\paragraph{\texttt{QMD}.} We train separate models for each of the 8 different tasks. We tune the hyperparameters of EB-GNN for each task based on the grid in Table~\ref{tab:hyp_QMD_grid}. We train for 500 epochs with a batch size of 1024 and evaluate on 10 different seeds.
	
		\begin{table}[ht]
		\centering
		\caption{Hyperparameter grid used on \texttt{QMD} and \texttt{QM9}.}
		\label{tab:hyp_QMD_grid}
		\begin{tabular}{ll}
			\toprule
			\textbf{Hyperparameter} & \textbf{Values} \\
			\midrule
			Embedding dimension & 128, 256 \\
			Dropout rate & 0, 0.2, 0.5 \\
			Number of message passing layers & 3, 4, 5 \\
			Pooling (graph-level tasks) & ${\cal T}_{\text{SUM}}$, ${\cal T}_{\text{MEAN}}$ \\
			Pooling (edge-level tasks) & \tablecell{sum directed embeddings $\rightarrow$ make undirected prediction,\\ make directed predictions $\rightarrow$ sum into undirected prediction} \\
			\bottomrule
		\end{tabular}
		
	\end{table}

	\paragraph{\texttt{QM9}.} We train separate models for each of the 11 different tasks. Initially, we repeated the same procedure as for \texttt{QMD}. However,  training with a significantly larger batch size than models in literature might harm our performance. Thus, after the initial hyperparameter sweep (Table~\ref{tab:hyp_QMD_grid}) we used the best hyperparameters and additionally tuned the batch size together with another pooling operation (Tab.~\ref{tab:hyp_QM9_small_grid}). We evaluate the best hyperparameter combination on 10 different seeds.
	
		\begin{table}[ht]
		\centering
		\caption{Smaller hyperparameter grid used on \texttt{QM9} after hyperparameter tuning with Tab.~\ref{tab:hyp_QMD_grid}.}
		\label{tab:hyp_QM9_small_grid}
		\begin{tabular}{ll}
			\toprule
			\textbf{Hyperparameter} & \textbf{Values} \\
			\midrule
			Embedding dimension & Best from Tab.~\ref{tab:hyp_QMD_grid} \\
			Dropout rate & Best from Tab.~\ref{tab:hyp_QMD_grid} \\
			Number of message passing layers & Best from Tab.~\ref{tab:hyp_QMD_grid}\\
			Batch size & 64, 1024 \\
			Pooling & \tablecell{${\cal T}_{\text{NODESUM}}$,\\ Best from Tab.~\ref{tab:hyp_QMD_grid}} \\
			\bottomrule
		\end{tabular}
	\end{table}

	Similar to \citet{zhou2023distancerestricted} and \citet{WL_Loopy}, we perform a speed evaluation on \texttt{QM9}. For this, we train our model on the training set with batch size 64 and measure the time it takes to train for a single epoch. Additionally, we also track the pre-processing time on the entire dataset. The results can be seen in Tab.~\ref{tab:runtime_complexity}. There are two issues with this type of evaluation. First, we (as well as previous work) compare runtime of models trained on different hardware. This is best noticed by the fact that our expressive EB-GNN is faster than the MPNN in Tab.~\ref{tab:runtime_complexity}. In our case, this is less of a problem because compared to 5$\ell$-GIN (the other best model on \texttt{QM9}), our GPU is significantly weaker (RTX 3080 vs RTX 3090 Ti/RTX A6000). Second, we believe that previous works included the initial processing time for the dataset in pre-processing. This includes time spent to generate the initial graphs which is needed for all models but can vary across hardware. We remedy this by reporting both the time spent for our pre-processing (70 seconds) as well as the time spent on preparing the initial graphs (30 seconds). 
	
	\begin{table}[ht]
		\centering
		\caption{Empirical time complexity for QM9 dataset; results from \citet{zhou2023distancerestricted} and \citet{WL_Loopy}.}
		\label{tab:runtime_complexity}
		\begin{tabular}{cccc}
			\toprule
			Model & Preprocessing [sec] & Training [sec/epoch]\\
			\midrule
			MPNN & $64$ & $45.3$ \\
			NestedGNN & $2\,354$ & $107.8$\\ 
			I2GNN & $5\,287$ & $209.9$\\
			2-DRFWL & $430$ & $141.9$\\
			5-$\ell$GIN & $444$ & $130.6$\\
			\midrule
			EB-GNN (ours) & 100 & 31 \\
			\bottomrule
		\end{tabular}
	\end{table}

	\paragraph{\texttt{MalNet-Tiny}.} As the graphs in \texttt{MalNet-Tiny} are very large, we did not perform any hyperparameter tuning. Our model uses mean pooling, has 5 message passing layers and an embedding dimension of 64. It is trained for 500 epochs with a batch size of 16 and evaluated on  5 different seeds. Note that this graph is \emph{directed}. For this we extended EB-GNN as follows. We only compute edge embeddings for directed edges that are part of the graph. Furthermore, for $\alpha$ and $\gamma$ only aggregate edges according to their direction, i.e., $\mathcal{N}$ are directed neighbors. For $\beta$, we adapt the computation of triangles to treat the entire graph as undirected. This can lead to aggregations of an edge $(u,v)$ where only the other direction $(v,u)$ is part of the graph. In this case, we aggregate over $(v,u)$, the direction that exists in the graph instead.

\section{Relationship to $\delta$-$k$-LWL}
\label{app:relationship-delta-k-lwl}

\citet{DBLP:conf/nips/0001RM20} introduced $\delta$-$k$-WL and its local variant $\delta$-$k$-WL as a variant of the $k$-WL test that is more efficient on sparse graphs, matching our motivation for the introduction of EB-1WL. The relationship between the two formalisms, especially for the case $k=2$, is interesting. The formalisms look superficially similar, while being conceptually quite far apart. We therefore elaborate on these differences explicitly here.

In $\delta$-$k$-LWL, colors $\mathsf{c}^{(\ell)}$  are assigned to $k$-tuples of vertices. Initially, $\mathsf{c}(\mathbf{v})$ is colored by the isomorphism type of $G[\mathbf{v}]$ (the induced subgraph of $G$ by vertices in $\mathbf{v}$). We refer to position $i$ of $\mathbf{v}$ as $v_i$. Moreover, let $\theta_i(\mathbf{v},x)$ denote the tuple obtained by replacing the $i$ component of $\mathbf{v}$ with $x$.
The colors are then updated iteratively according to
\begin{align}
	\mathsf{c}^{(\ell+1)}(\mathbf{v}) = \big( &\mathsf{c}^{(\ell)}(\mathbf{v}, \\
	& \{\!\{ \mathsf{c}^{(\ell)}(\theta_1(\mathbf{v}, x)) \mid x \in N(v_1) \}\!\}, \\
	& \cdots \\
	& \{\!\{ \mathsf{c}^{(\ell)}(\theta_k(\mathbf{v}, x)) \mid x \in N(v_k) \}\!\} \ \big)
\end{align}

On the surface the core update part, gathering the multisets over neighbors of
each component, seems similar to parts (2) and (4) in our update function for
$\mathsf{eb}$. However, there are key differences: In $\delta$-$k$-LWL the update considers vectors where the $i$-th component is replaced by neighbors of the $i$-th component, i.e., we replace $v_i$ with elements in $N(v_i)$.
In EB-1WL we, in a sense, do the opposite. We take the edges incident to each $v_i$, that is we form the tuples $(v_i, x)$ for $x\in N(v_i)$. Instead of replacing $v_i$ by its neighbors, we look at the tuples formed by $v_i$ with its neighbors.

We illustrate this with the following simple example graph:
\[
\begin{tikzpicture}[baseline=-0.5ex, node distance=12mm]
	\tikzstyle{v}=[circle,draw,inner sep=1pt,minimum size=6pt]
	\node[v] (A) {}; \node[below=0pt of A] {$a$};
	\node[v] (B) [right of=A] {}; \node[below=0pt of B] {$b$};
	\node[v] (C) [right of=B] {}; \node[below=0pt of C] {$c$};
	\draw (A) -- (B) -- (C);
\end{tikzpicture}
\]

The $\delta$-2-WL update for $\mathsf{c}((a,b))$ would update based on the pair of multisets  
\(\{\!\{C(a,x) \mid x \in N(b)\}\!\}\) and \(\{C(x,c) \mid x \in N(a)\}\), that is both times \(\{\!\{C(a,c)\}\!\}\).  
That is, we replace each component with all its neighbors, i.e., a classic \(k\)-WL style update but limited to local neighborhoods.

The seemingly similar part of our update (that is only (2) and (4), ignoring (3)) would update based on  
\(\{C(a,x) \mid x \in N(a)\}\) (2) and \(\{C(b,x) \mid x \in N(b)\}\).  
So \(\{C(a,b)\}\) and \(\{C(b,c)\}\), i.e., the update is based on the edges incident to edge \((a,b)\) (and the edge itself).

Another noteworthy difference is in the computational complexity of a layer update.
An update for $\delta$-$2$-WL requires   
\(O(n^2 d)\) time, where \(d\) is the degree of the graph. As each $n^2$ pairs
has to access $k\cdot d$ colors of neighbors. Intuitively, this bound is
reasonably tight as neither part --- the number of tuples, or the access of
$k\cdot d$ tuples in the update ---  can be avoided.
In contrast, our bound of \(O(\alpha m)\) promises significantly better
performance, especially in sparse graphs as they have \(m \ll n^2\) and
\(\alpha \ll d\).

	\section{EB-1WL vs 1WL with Triangle Information}
\label{app:wl-triangle}

The important role of triangles in EB-1WL naturally motivates a comparison with MPGNNs that have triangle counts injected. 

Here we provide an example demonstrating that EB-1WL is strictly more expressive than 1WL with triangle counts added on node and edge level. \Cref{fig:stupidex} shows two graphs from the BREC dataset that are indistinguishable by 1WL with such triangle counts but distinguishable by EB-1WL.  \Cref{fig:code} gives example code to verify the indsitinguishability by 1WL. Moreover, the graphs exhibit a different number of homomorphisms from the 4-cycle with a chord and by \Cref{theo:counts} is therefore distinguished by EB-1WL.

\begin{figure}[h]
	\centering
	\begin{subfigure}{0.45\textwidth}
		\centering
		\scalebox{.7}{
			\begin{tikzpicture}[>=stealth, node distance=2cm]
				\node (n0) at (5.00, 0.00) [circle,draw] {0};
				\node (n3) at (4.05, 2.94) [circle,draw] {3};
				\node (n4) at (1.55, 4.76) [circle,draw] {4};
				\node (n7) at (-1.55, 4.76) [circle,draw] {7};
				\node (n8) at (-4.05, 2.94) [circle,draw] {8};
				\node (n9) at (-5.00, -0.00) [circle,draw] {9};
				\node (n1) at (-4.05, -2.94) [circle,draw] {1};
				\node (n5) at (-1.55, -4.76) [circle,draw] {5};
				\node (n6) at (1.55, -4.76) [circle,draw] {6};
				\node (n2) at (4.05, -2.94) [circle,draw] {2};
				\draw (n0) -- (n3);
				\draw (n0) -- (n4);
				\draw (n0) -- (n7);
				\draw (n0) -- (n8);
				\draw (n0) -- (n9);
				\draw (n3) -- (n1);
				\draw (n3) -- (n5);
				\draw (n3) -- (n6);
				\draw (n3) -- (n7);
				\draw (n3) -- (n8);
				\draw (n3) -- (n9);
				\draw (n4) -- (n2);
				\draw (n4) -- (n5);
				\draw (n4) -- (n6);
				\draw (n4) -- (n7);
				\draw (n4) -- (n8);
				\draw (n7) -- (n1);
				\draw (n7) -- (n2);
				\draw (n7) -- (n5);
				\draw (n7) -- (n9);
				\draw (n8) -- (n1);
				\draw (n8) -- (n2);
				\draw (n8) -- (n6);
				\draw (n9) -- (n1);
				\draw (n9) -- (n2);
				\draw (n9) -- (n5);
				\draw (n9) -- (n6);
				\draw (n1) -- (n5);
				\draw (n1) -- (n6);
				\draw (n5) -- (n2);
				\draw (n6) -- (n2);
			\end{tikzpicture}
		}
		\caption{$G_0$}
	\end{subfigure}
	\hfill
	\begin{subfigure}{0.45\textwidth}
		\centering
		\scalebox{.7}{
			\begin{tikzpicture}[>=stealth, node distance=2cm]
				\node (n0) at (5.00, 0.00) [circle,draw] {0};
				\node (n3) at (4.05, 2.94) [circle,draw] {3};
				\node (n4) at (1.55, 4.76) [circle,draw] {4};
				\node (n5) at (-1.55, 4.76) [circle,draw] {5};
				\node (n6) at (-4.05, 2.94) [circle,draw] {6};
				\node (n7) at (-5.00, -0.00) [circle,draw] {7};
				\node (n9) at (-4.05, -2.94) [circle,draw] {9};
				\node (n1) at (-1.55, -4.76) [circle,draw] {1};
				\node (n8) at (1.55, -4.76) [circle,draw] {8};
				\node (n2) at (4.05, -2.94) [circle,draw] {2};
				\draw (n0) -- (n3);
				\draw (n0) -- (n4);
				\draw (n0) -- (n5);
				\draw (n0) -- (n6);
				\draw (n0) -- (n7);
				\draw (n0) -- (n9);
				\draw (n3) -- (n1);
				\draw (n3) -- (n5);
				\draw (n3) -- (n6);
				\draw (n3) -- (n8);
				\draw (n3) -- (n9);
				\draw (n4) -- (n1);
				\draw (n4) -- (n6);
				\draw (n4) -- (n7);
				\draw (n4) -- (n8);
				\draw (n4) -- (n9);
				\draw (n5) -- (n1);
				\draw (n5) -- (n2);
				\draw (n5) -- (n7);
				\draw (n5) -- (n8);
				\draw (n6) -- (n2);
				\draw (n6) -- (n7);
				\draw (n6) -- (n8);
				\draw (n6) -- (n9);
				\draw (n7) -- (n1);
				\draw (n7) -- (n2);
				\draw (n7) -- (n9);
				\draw (n9) -- (n1);
				\draw (n9) -- (n2);
				\draw (n1) -- (n8);
				\draw (n8) -- (n2);
			\end{tikzpicture}
		}
		\caption{$G_1$}
	\end{subfigure}
	
	\caption{The first graph pair in the Extension graphs of the BREC dataset~\cite{BREC}.}
	\label{fig:stupidex}
\end{figure}

\begin{figure}[ht]
	\centering
	\begin{verbatim}
		from BRECDataset import BRECDataset
		import networkx as nx
		
		def edge_idx_to_nx(ei):
		g = nx.Graph()
		for a,b in zip(ei[0], ei[1]):
		g.add_edge(int(a), int(b))
		return g
		
		d = BRECDataset("BREC_data_all")
		
		g0 = edge_idx_to_nx(d[160*32*2 : (160+1)*32*2][0].edge_index)
		g1 = edge_idx_to_nx(d[160*32*2 : (160+1)*32*2][1].edge_index)
		
		
		def edge_tri(g, multiset=False):
		vt = nx.triangles(g)
		et = {m: vt[m[0]]+vt[m[1]] for m in g.edges()}
		if multiset:
		return list(sorted(et.values()))
		else:
		return et
		
		nx.set_node_attributes(g0, nx.triangles(g0), 'c3')
		nx.set_node_attributes(g1, nx.triangles(g1), 'c3')
		
		nx.set_edge_attributes(g0, edge_tri(g0), 'ec3')
		nx.set_edge_attributes(g1, edge_tri(g1), 'ec3')
		
		h1 = nx.weisfeiler_lehman_graph_hash(g0, node_attr='c3', edge_attr='ec3')
		h2 = nx.weisfeiler_lehman_graph_hash(g1, node_attr='c3', edge_attr='ec3')
		
		print(h1, h2, h1 == h2)
	\end{verbatim}
	\caption{Example code to verify the indistinguishability of $G_0$ and $G_1$ by 1WL with triangle counts on edges and vertices.}
	\label{fig:code}
\end{figure}

\end{document}